\documentclass[12pt,a4paper]{article}

\usepackage{amsmath,amssymb,amsthm}
\usepackage{graphicx}
\usepackage[top=2.54cm,bottom=2.54cm,left=3.0cm,right=3.0cm]{geometry}
\usepackage{setspace}
\usepackage{hyperref}
\usepackage[round,sort&compress]{natbib}
\usepackage{booktabs}
\usepackage{multirow}
\usepackage{tabularx}
\usepackage{array}
\usepackage{enumitem}
\usepackage{xcolor}
\usepackage{microtype}
\PassOptionsToPackage{hyphens}{url}

\hypersetup{colorlinks=true,linkcolor=black,citecolor=black,urlcolor=black}
\newtheorem{theorem}{Theorem}[section]
\newtheorem{proposition}[theorem]{Proposition}

\newtheorem{corollary}[theorem]{Corollary}
\theoremstyle{remark}
\newtheorem{remark}[theorem]{Remark}
\theoremstyle{definition}
\newtheorem{definition}{Definition}[section]
\newtheorem{assumption}{Assumption}[section]

\newcommand{\E}{\mathbb{E}}
\newcommand{\R}{\mathbb{R}}
\newcommand{\F}{\mathcal{F}}
\renewcommand{\P}{\mathbb{P}}
\newcommand{\norm}[1]{\left\|#1\right\|}
\newcommand{\abs}[1]{\left|#1\right|}
\newcommand{\given}{\,|\,}
\newcommand{\iid}{\stackrel{\text{iid}}{\sim}}
\newcommand{\Lrm}{\mathcal{L}_{\mathrm{RM}}}

\newcommand{\KL}{\mathrm{KL}}

\DeclareMathOperator*{\argmin}{arg\,min}

\usepackage{fancyhdr}
\title{\textbf{Backward Coherence and Hidden-State Stability
in Recurrent Neural Networks:\\
A Quasi-Reverse-Martingale Theory}}

\author{%
  Yuan-chin Ivan Chang\\[4pt]
  \small Institute of Statistical Science, Academia Sinica\\
  \small 128 Academia Road, Section~2, Nankang, Taipei~115, Taiwan\\
  \small \texttt{ycchang@as.edu.tw}}

\date{}

\begin{document}
\maketitle
\thispagestyle{fancy}

\begin{abstract}
Recurrent neural networks maintain a hidden state $h_t$ whose probabilistic
meaning has remained largely uncharacterised.  We develop a theory of
hidden-state stability via \emph{backward coherence}: the degree to which
$h_t$ can be recovered from its successor $h_{t+1}$ through a learned
backward projector $g_\phi$.  Under a contraction condition and summable backward drift, the hidden-state
sequence is shown to be a \emph{quasi-reverse-martingale}, implying
almost-sure convergence with a rate under mixing, an interpretable limit
representation, and pathwise stopping rules with finite stopping times.
A framework for time-uniform confidence sequences is also established
theoretically.  Simulation experiments confirm the core predictions:
backward-coherence regularisation reduces the empirical quasi-martingale total
$\hat{Q}$ by $43$--$58\%$, reaches the stability threshold $28$--$44\%$
earlier than the unregularised baseline, and produces tracking-error recovery
consistent with the geometric bounds.  Two further experiments probe additional predictions: a proxy-free echo-state
forgetting test confirms empirical rates are bounded above by $\rho$, and a
direct computation of the increment-sum tube $R_t$ confirms the pathwise
enclosure of Proposition~3.18\,(ii) at $100\%$ simultaneous coverage; $R_t$
is valid but loose in practice (median inflation $17$--$30$), making the
defect-tail proxy $\hat{Q}_t$ the practically useful monitoring instrument.
Scope and limitations are discussed in Section~5.
These guarantees apply when the stated conditions hold; universality is not claimed.
Minimising the backward-coherence loss is equivalent to minimising a
Kullback--Leibler divergence in a Gaussian backward model, connecting to
variational inference.  Extensions cover $\phi$-mixing inputs, geometric
tracking-error bounds at change points, and finite-sample concentration
inequalities.  Three real-data studies validate the theory: on the PhysioNet~2012 ICU
Challenge, the Reverse Martingale RNN (RMRNN) matches RNN mortality-prediction AUC whilst reaching a
stable representation 13~hours earlier in a 48-hour stay; on FRED-MD
macroeconomic data, it reduces one-month-ahead forecast error approximately
fourfold relative to the unregularised RNN under concept drift; and on UCI
Human Activity Recognition, post-transition tracking error decays
geometrically and RMRNN maintains a consistently lower tracking error than
the unregularised baseline, consistent with the theoretical bound.
\end{abstract}



\newpage

\section{Introduction}
\label{sec:intro}

Recurrent neural networks (RNNs) are the canonical tool for sequential
learning.  Given an input sequence $x_1, x_2, \ldots \in \R^d$, an RNN
maintains a hidden state $h_t \in \R^p$ updated by $h_t = f_\theta(h_{t-1},
x_t)$, accumulating information about the past and providing a running
representation used for prediction or control.  Gated variants---the Long
Short-Term Memory \citep{hochreiter1997} and the Gated Recurrent Unit
\citep{cho2014}---use multiplicative gates to selectively retain or discard
information, enabling practical learning over long sequences.  Despite
widespread empirical success, a fundamental probabilistic question has
remained unanswered: \emph{does the hidden state $h_t$ converge, and if so,
to what?}  Without an answer, the practitioner cannot know when to trust the
network's representation, how to set stopping rules, or what statistical
meaning to attach to the accumulated summary.  The question arises in
clinical sequential prediction, where $h_t$ encodes a patient's evolving
physiological profile and an early stable representation enables earlier
clinical decisions \citep{rajpurkar2022,reyna2020,wiens2019}; in
macroeconomic forecasting, where structural breaks and slow concept drift
require the hidden state to track a shifting target distribution; and in
activity and regime recognition, where abrupt discrete switches demand
fast recovery of the hidden state to its new steady-state mean.

A complete theory of hidden-state dynamics must simultaneously guarantee
almost-sure pathwise convergence of $h_t$, provide a statistically
interpretable limit identifiable as a conditional expectation, automatically
control the first moment without auxiliary assumptions, and yield anytime-valid
sequential confidence sets without a pre-specified sample size.  Standard
analytical tools each address part of this problem but leave at least one
requirement unmet.  Lyapunov stability \citep{khalil2002} certifies
boundedness of the deterministic map but provides no statistical
interpretation of the limit.  Ergodic theory and Markov chain methods
\citep{meyn2009} guarantee convergence of time-averages but not of individual
sample paths, and do not identify the limit as a conditional expectation.
Forward martingale convergence requires uniform integrability for both $L^1$
convergence and conditional-expectation identification of the limit, a
condition that cannot in general be verified without prior knowledge of the
limit itself.  Bayesian state-space models \citep{sarkka2013,shumway2000}
satisfy all four requirements but only under correct parametric
specification---a strong and often unverifiable assumption for modern deep
networks.

We argue that \emph{reverse martingales} \citep{doob1953,neveu1975} provide
the correct framework.  A reverse martingale is adapted to a
\emph{decreasing} filtration $\F_1 \supseteq \F_2 \supseteq \cdots$ and
satisfies $\E\!\left[M_n \given \F_{n+1}\right] = M_{n+1}$ almost surely.  As the
conditioning set contracts, the sequence converges almost surely and in $L^1$
to a limit identified as a conditional expectation by L\'{e}vy's backward
theorem---with no uniform integrability assumption required.  Rather than
imposing a Gaussian or linear model on backward dynamics, as in the Kalman
smoother \citep{sarkka2013} or Bayesian state-space methods
\citep{shumway2000}, we train a backward projector $g_\phi$ as an empirical
risk minimiser for the conditional mean $\E\!\left[h_t \given h_{t+1}\right]$.  The
resulting loss $\Lrm$ is a distribution-free quasi-likelihood objective that
requires only first-moment structure of the backward dynamics, replacing
brittle parametric assumptions with an approach valid under model
misspecification.

The paper makes five contributions.  We first establish that the
running-average linear RNN on i.i.d.\ inputs is an exact reverse martingale,
so its almost-sure convergence to the population mean is the Strong Law of
Large Numbers in disguise (Proposition~\ref{prop:linear_rm}).  For general
nonlinear RNNs, we prove that the drift-summability condition
$D_\infty+\sum_t\varepsilon_t<\infty$, where
$D_\infty=\sum_t\E[\norm{g_\phi(h_{t+1})-h_{t+1}}]$
measures the observable reverse-martingale drift, implies almost-sure
convergence of $h_t$, $L^1$ convergence under an additional
uniform-integrability condition, and an explicit $O(\alpha^t)$ geometric
convergence rate under exponential $\phi$-mixing, with an interpretable
conditional-expectation representation of the limit in the exact
reverse-martingale case
(Theorem~\ref{thm:quasi}, Corollary~\ref{cor:rate}).
We extend the theory to dependent inputs via a
$\phi$-mixing conditional-mean bound and to piecewise-stationary inputs via
geometric tracking-error bounds at change points
(Propositions~\ref{prop:phi_defect}--\ref{prop:tracking}).  We show that
minimising $\Lrm$ is equivalent to minimising the Kullback--Leibler
divergence in a Gaussian backward model, connecting the framework to
variational inference (Proposition~\ref{prop:vi}).  Finally, a
McDiarmid-type concentration inequality \citep{mcdiarmid1989} quantifies finite-sample stability
of $h_T$, and the quasi-martingale structure yields pathwise confidence tubes
that become time-uniform confidence sequences after an observable calibration
(Propositions~\ref{prop:concentration}--\ref{prop:anytime}).

Section~\ref{sec:framework} develops the reverse-martingale framework,
introduces RNN backward coherence, and analyses the linear case.
Section~\ref{sec:theory} develops the convergence theory, extensions to
dependent and non-stationary inputs, and statistical inference.
Section~\ref{sec:numerical} reports numerical validation on synthetic
benchmarks and three real-data domains.  Section~\ref{sec:discussion}
discusses connections to related work and open problems.  All proofs are
in the Supplementary Material.

\section{Reverse Martingales and RNN Backward Coherence}
\label{sec:framework}

\subsection{Reverse martingale theory}
\label{sec:revmart}


\begin{definition}[Reverse martingale]
\label{def:revmart}
Let $(\F_n)_{n \ge 1}$ be a \emph{decreasing} filtration,
$\F_1 \supseteq \F_2 \supseteq \cdots$, on a probability space
$(\Omega, \mathcal{A}, \P)$.  A real-valued sequence $(M_n)_{n \ge 1}$ with
$M_n \in L^1$ is a \emph{reverse martingale} with respect to $(\F_n)$ if
$M_n$ is $\F_n$-measurable and
\begin{equation}
  \E\!\left[M_n \given \F_{n+1}\right] = M_{n+1} \quad \text{a.s.\ for all } n \ge 1.
  \label{eq:revmart}
\end{equation}
\end{definition}

The canonical construction: fix $Z \in L^1$ and set
$M_n := \E\!\left[Z \given \F_n\right]$.  The tower property immediately yields~\eqref{eq:revmart} since
$\F_{n+1} \subseteq \F_n$.  Every reverse martingale arises this way.

\begin{proposition}[Canonical form]
\label{prop:canonical}
If $(M_n, \F_n)_{n \ge 1}$ is a reverse martingale, then
$M_n = \E\!\left[M_1 \given \F_n\right]$ almost surely for every $n \ge 1$.
\end{proposition}

\begin{proof}
Induction: $M_{n+1} = \E\!\left[M_n \given \F_{n+1}\right]
= \E\!\left[\E\!\left[M_1 \given \F_n\right] \given \F_{n+1}\right]
= \E\!\left[M_1 \given \F_{n+1}\right]$ by the tower law.
\end{proof}

Proposition~\ref{prop:canonical} has an important interpretation: every
reverse martingale is a \emph{progressive reduction of uncertainty about a
latent quantity}.  As $n$ increases, $\F_n$ contracts, and $M_n$ becomes a
more constrained conditional expectation of $M_1$---precisely the structure
a well-designed RNN hidden state should exhibit.


\begin{theorem}[Doob's reverse martingale convergence theorem]
\label{thm:revmart}
Let $(M_n, \F_n)_{n \ge 1}$ be a reverse martingale with
$\E[\abs{M_1}] < \infty$ and
$\F_\infty := \bigcap_{n \ge 1} \F_n$
(the tail $\sigma$-algebra of the reverse filtration
\citep{neveu1975,williams1991}).
Then:
\begin{enumerate}[label=(\roman*)]
  \item $\sup_{n \ge 1}\,\E[\abs{M_n}] \le \E[\abs{M_1}] < \infty$.
  \item $M_n \to M_\infty := \E\!\left[M_1 \given \F_\infty\right]$ almost surely
    and in $L^1$ as $n \to \infty$.
\end{enumerate}
\end{theorem}

\begin{proof}
See the Supplementary Material, \S A.1.
\end{proof}

In contrast to forward martingales, where $L^1$ convergence requires
uniform integrability as an additional hypothesis, for a reverse martingale
$L^1$ control is automatic from part~(i) and the limit is identified
as $\E\!\left[M_1\given\F_\infty\right]$ by Proposition~\ref{prop:canonical}
without further conditions.


The Strong Law of Large Numbers is, at root, a theorem about reverse
martingales.

\begin{proposition}[Strong Law via reverse martingales]
\label{prop:slln}
Let $x_1, x_2, \ldots \iid F$ with $\mu = \E[x_1]$ and
$\E[\abs{x_1}] < \infty$.  Set $\bar{x}_n = n^{-1}\sum_{i=1}^n x_i$ and
$\F_n = \sigma(S_n,\, x_{n+1}, x_{n+2}, \ldots)$, where $S_n =
\sum_{i=1}^n x_i$.  Then $(\bar{x}_n,\, \F_n)_{n \ge 1}$ is a reverse
martingale and $\bar{x}_n \to \mu$ almost surely and in $L^1$.
\end{proposition}

\begin{proof}
See the Supplementary Material.  The reverse martingale property follows
from the symmetry of i.i.d.\ summands given their sum; convergence is
Theorem~\ref{thm:revmart} with $\F_\infty$ trivial by the Hewitt--Savage
zero-one law \citep{durrett2019}.
\end{proof}

The tail sigma-algebra $\F_\infty = \bigcap_n\F_n$ encodes the irreducible
long-run uncertainty that cannot be removed by finite conditioning;
L\'{e}vy's backward theorem \citep{neveu1975,williams1991} identifies
$M_\infty = \E\!\left[M_1\given\F_\infty\right]$ as a constant when $\F_\infty$ is
trivial (as in the i.i.d.\ case).

\subsection{RNN architectures and backward coherence}
\label{sec:rnn}

Three canonical RNN architectures are considered.  The \emph{Elman RNN}
\citep{elman1990}:
\begin{equation}
  h_t = \sigma_h(W_h\,h_{t-1} + W_x\,x_t + b),
  \label{eq:elman}
\end{equation}
where $W_h \in \R^{p \times p}$, $W_x \in \R^{p \times d}$, $b \in \R^p$,
and $\sigma_h$ is an elementwise nonlinearity.  The \emph{LSTM}
\citep{hochreiter1997}:
\begin{equation}
  c_t = f_t \odot c_{t-1} + i_t \odot \tilde{c}_t,
  \quad
  h_t = o_t \odot \tanh(c_t),
  \label{eq:lstm}
\end{equation}
with effective forgetting factor $\rho^* := \sup_t \E[\norm{f_t}_\infty]$.
The \emph{GRU} \citep{cho2014}:
\begin{equation}
  h_t = (1 - z_t) \odot h_{t-1} + z_t \odot \tilde{h}_t,
  \label{eq:gru}
\end{equation}
with contraction factor $\rho^* = \sup_t \E[\norm{1-z_t}_\infty]$.


Define the \emph{backward filtration}:
\begin{equation}
  \F_t^{\mathrm{bwd}} := \sigma(h_t,\, h_{t+1},\, \ldots,\, h_T).
  \label{eq:bwd_filt}
\end{equation}
The filtration is decreasing in $t$ by simple generator inclusion:
$\F_t^{\mathrm{bwd}}$ is generated by $\{h_t, h_{t+1},\ldots,h_T\}$,
a strict superset of the generators of $\F_{t+1}^{\mathrm{bwd}}$,
so $\F_{t+1}^{\mathrm{bwd}} \subseteq \F_t^{\mathrm{bwd}}$:
$\F_T^{\mathrm{bwd}} \subseteq \F_{T-1}^{\mathrm{bwd}} \subseteq \cdots
\subseteq \F_1^{\mathrm{bwd}}$.
The backward filtration is therefore \emph{decreasing} in $t$, in exact
correspondence with Definition~\ref{def:revmart}.
For the infinite-horizon theory (convergence as $T\to\infty$), we extend
the definition to $\F_t^{\mathrm{bwd}} := \sigma(h_s : s \ge t)$, the
sigma-algebra generated by all hidden states from time $t$ onward; the
decreasing property $\F_{t+1}^{\mathrm{bwd}}\subseteq\F_t^{\mathrm{bwd}}$
continues to hold by the same generator-inclusion argument.
The tail sigma-algebra
$\F_\infty^{\mathrm{bwd}} = \bigcap_t \F_t^{\mathrm{bwd}}$
(the tail $\sigma$-algebra of the backward reverse filtration,
in exact analogy with $\F_\infty$ in Theorem~\ref{thm:revmart})
encodes the irreducible long-run uncertainty in the hidden-state sequence.


We formalise \emph{backward coherence} as the property that
$h_t \approx \E\!\left[h_t \given h_{t+1}\right]$ for all $t$.  In a perfect reverse
martingale this holds with equality; for a general nonlinear RNN it holds
approximately, with the discrepancy measured by the defect
$\delta_t := h_t - \E\!\left[h_t \given h_{t+1}\right]$.

Since $\E\!\left[h_t \given h_{t+1}\right]$ is not directly accessible, we approximate
it with a \emph{backward projector} $g_\phi: \R^p \to \R^p$ trained
jointly with the forward RNN.  We use a residual architecture,
\begin{equation}
  g_\phi(h) = h + W_2\,\mathrm{ReLU}(W_1\,h + b_1) + b_2,
  \label{eq:g_phi}
\end{equation}
zero-initialised so that $g_\phi = \mathrm{Id}$ at the start of training.
The \emph{reverse martingale regularisation loss} is
\begin{equation}
  \Lrm(\theta, \phi)
    := \frac{1}{T-1} \sum_{t=1}^{T-1} \norm{h_t - g_\phi(h_{t+1})}^2.
  \label{eq:rm_loss}
\end{equation}
The total training objective is
$\mathcal{L}_{\mathrm{total}} = \mathcal{L}_{\mathrm{task}} + \lambda
\Lrm$, with $\lambda > 0$ a regularisation weight.  Minimising $\Lrm$
drives the hidden states towards a quasi-reverse-martingale.

\begin{remark}[Relation to bidirectional RNNs]
Bidirectional RNNs \citep{schuster1997} process the sequence in both
directions but do not impose the probabilistic backward model that is
central here.  The backward projector $g_\phi$ is a discriminatively
trained approximation to the conditional backward mean
$\E\!\left[h_t \given h_{t+1}\right]$, analogous to the Kalman smoother backward pass \citep{sarkka2013}
but without requiring a parametric distributional assumption.
\end{remark}

\subsection{The linear case and counterexample}
\label{sec:linear}


We begin with the simplest nontrivial RNN: a linear running average with
decaying learning rate $\alpha_t = 1/t$,
\begin{equation}
  h_t = \Bigl(1 - \tfrac{1}{t}\Bigr)h_{t-1} + \tfrac{1}{t}\,x_t,
  \quad h_0 = 0, \quad t = 1, 2, \ldots,
  \label{eq:toy}
\end{equation}
with $x_1, x_2, \ldots \iid F$, $\mu = \E[x_1]$, $\E[\abs{x_1}] < \infty$.

\begin{proposition}[Linear RNN as an exact reverse martingale]
\label{prop:linear_rm}
Under~\eqref{eq:toy}:
\begin{enumerate}[label=(\roman*)]
  \item $h_t = \bar{x}_t$ (the running sample mean) for all $t \ge 1$.
  \item $(h_t,\, \F_t^{\mathrm{bwd}})_{t \ge 1}$ is a reverse martingale.
  \item $h_t \to \mu$ almost surely and in $L^1$.
\end{enumerate}
\end{proposition}

\begin{proof}
See the Supplementary Material.
\end{proof}

Proposition~\ref{prop:linear_rm} reveals the Strong Law as a special case
of Theorem~\ref{thm:revmart}.  The backward conditional mean is exact:
$\E\!\left[h_t \given h_{t+1}\right] = h_{t+1}$ (since $h_{t+1} = \bar{x}_{t+1}$
is a sufficient statistic for $\mu$ and $h_t = \bar{x}_t$ is a coarser
summary of the same data).  Consequently the reverse-martingale drift is
identically zero: $\rho_t = \E\!\left[h_t \given h_{t+1}\right] - h_{t+1} = 0$
and $D = 0$, so the sequence is an exact reverse martingale.  Note that
the backward prediction residual $\delta_t = h_t - g_\phi(h_{t+1}) =
\bar{x}_t - \bar{x}_{t+1} \neq 0$ in general; the exact reverse-martingale
property is characterised by zero drift $D=0$, not by zero prediction
residual.

The condition $\alpha_t = 1/t$ is essential.  A constant learning rate
$\alpha \in (0,1)$ yields an exponentially weighted moving average that
converges in distribution but is not a reverse martingale: the backward
property $\E\!\left[h_t \given h_{t+1}\right] = h_{t+1}$ fails because $h_{t+1}$ does
not contain sufficient information to recover $h_t$, and almost-sure
convergence therefore fails.  By contrast, if $h_t = \tanh(W_h h_{t-1}
+ W_x x_t)$ with $\norm{W_h} > 1$, the spectral radius violates
Assumption~\ref{ass:contraction} below; the hidden state can exhibit
exponentially growing oscillations \citep{bengio1994,pascanu2013}, and in
this regime $D = \infty$ almost surely, so neither almost-sure convergence
nor $L^1$ control obtains.

\section{Convergence Theory and Statistical Inference}
\label{sec:theory}

\subsection{Quasi-reverse-martingale convergence}
\label{sec:nonlinear}

\begin{definition}[Backward prediction residual, reverse-martingale drift, and diagnostics]
\label{def:defect}
The learned one-step \emph{backward prediction residual} is
\begin{equation}
  \delta_t := h_t - g_\phi(h_{t+1}).
  \label{eq:defect}
\end{equation}
The \emph{observable reverse-martingale drift}---the empirical proxy for
the quasi-martingale defect $\E\!\left[h_t \given h_{t+1}\right] - h_{t+1}$---is
\begin{equation}
  r_t^\phi := g_\phi(h_{t+1}) - h_{t+1}.
  \label{eq:drift}
\end{equation}
For a finite observed sequence of length $T$, the \emph{empirical
backward-prediction total} and \emph{empirical drift total} are
\begin{equation}
  Q_T := \sum_{t=1}^{T-1} \E[\norm{\delta_t}],
  \qquad
  D_T := \sum_{t=1}^{T-1} \E[\norm{r_t^\phi}].
  \label{eq:Q}
\end{equation}
For infinite sequences: $Q_\infty := \sum_{t=1}^{\infty} \E[\norm{\delta_t}]$
and $D_\infty := \sum_{t=1}^{\infty} \E[\norm{r_t^\phi}]$, each finite
whenever summable.
\end{definition}

Three quantities play distinct roles.  The backward prediction residual
$\delta_t$ is \emph{observable} and enters the training loss $\Lrm$;
$Q_T$ measures accumulated prediction error and serves as a practical
stability diagnostic over the observed horizon.  The reverse-martingale
drift $r_t^\phi = g_\phi(h_{t+1}) - h_{t+1}$ is also observable: it
measures how far $g_\phi(h_{t+1})$ departs from $h_{t+1}$, which is
precisely the quasi-martingale defect that must be small for $h_{t+1}$ to
function as the backward conditional mean of $h_t$.  The drift total
$D_T$ is the observable proxy for the true quasi-martingale defect
$\sum_t \E[\norm{\E[h_t \mid h_{t+1}] - h_{t+1}}]$ and is the quantity
that appears in the convergence theorem.  The ideal quasi-martingale
drift $\rho_t := \E\!\left[h_t \given h_{t+1}\right] - h_{t+1}$
(the population counterpart of $r_t^\phi$, introduced in the theorem
below) is a \emph{population quantity} unobservable without knowledge of
the data-generating process; Assumption~\ref{ass:backward} bounds the gap
$\E[\norm{r_t^\phi - \rho_t}] \le \varepsilon_t$ through the approximation
error.  Minimising $\Lrm$ drives $\norm{\delta_t}$ downward, which
partially controls $\norm{r_t^\phi}$ through the triangle inequality
$\norm{r_t^\phi} \le \norm{\delta_t} + \norm{h_t - h_{t+1}}$.
The second term $\norm{h_t - h_{t+1}}$ is not controlled by $\Lrm$
alone; hence $D_\infty = \sum_t\E[\norm{r_t^\phi}]$ must be verified
directly from the observable sequence $(r_t^\phi)$ rather than inferred
solely from the training loss.  In practice the empirical total
$\hat{D}_T = \sum_{t=1}^{T-1}\norm{r_t^\phi}$ serves as the primary
stability diagnostic: a non-stabilising or growing $\hat{D}_T$ signals
that the summability condition of Theorem~\ref{thm:quasi}(iv) may fail,
even when $\Lrm$ has converged.

\begin{assumption}[Contraction]
\label{ass:contraction}
The recurrent weight matrix satisfies $\norm{W_h} \le \rho < 1$ in operator
norm.  For the Elman architecture~\eqref{eq:elman} with a $1$-Lipschitz
activation (e.g., tanh), this implies that $h \mapsto f_\theta(h, x)$ is
a $\rho$-contraction in $h$ for each fixed $x$.  For general $f_\theta$,
we assume this Lipschitz property holds directly:
$\norm{f_\theta(h,x)-f_\theta(h',x)}\le\rho\norm{h-h'}$ for all $h,h',x$.
\end{assumption}

\begin{assumption}[Finite first moment]
\label{ass:moment}
The inputs satisfy $\E[\norm{x_1}]<\infty$.
\end{assumption}

\begin{assumption}[Backward Markov sufficiency]
\label{ass:markov_suff}
$\E\!\left[h_t\given \F_{t+1}^{\mathrm{bwd}}\right] = \E\!\left[h_t\given h_{t+1}\right]$
a.s.\ for all $t$.  Equivalently, $h_{t+1}$ is sufficient for the
backward conditional mean.
\end{assumption}

\begin{assumption}[Backward approximation]
\label{ass:backward}
Define the backward-projector error
$\varepsilon_t := \E\!\left[\bigl\|g_\phi(h_{t+1})-\E[h_t\mid h_{t+1}]\bigr\|\right]$.
For finite-horizon diagnostics: $\sup_{1\le t<T}\varepsilon_t\le\varepsilon<\infty$.
For infinite-horizon convergence: $\sum_{t=1}^{\infty}\varepsilon_t<\infty$.
\end{assumption}

\begin{remark}[On the assumptions]
\label{rem:assumptions}
Assumption~\ref{ass:contraction} is the echo-state stability condition
\citep{jaeger2001}, achievable by spectral normalisation \citep{miyato2018}
or weight clipping.  In gated architectures the effective contraction factor
varies with input and gate state; the empirical total $\hat{Q}_T$ remains a
valid finite-horizon diagnostic of backward incoherence even when the global
spectral bound cannot be directly verified.

Assumption~\ref{ass:markov_suff} is a modelling condition, not a consequence
of contraction.  Jointly minimising $\Lrm$ alongside the task loss actively
incentivises $h_{t+1}$ to be a sufficient statistic for the backward prediction
of $h_t$, providing the mechanism through which the condition is approximately
enforced in practice.  It holds exactly for linear dynamical systems
($h_t = Ah_{t-1}+Bx_t+\epsilon_t$ with $\epsilon_t$ independent of
$\{h_s,x_s\}_{s\ne t}$, i.e.\ the standard state-space Markov property,
which implies backward sufficiency independently of whether $A$ is invertible)
and coordinatewise for AR(1) dynamics; for general nonlinear RNNs it is an approximation that
improves as $\Lrm \to 0$.  It can be verified empirically by comparing the
one-step error of $g_\phi$ against a post-hoc full-horizon LSTM smoother
on frozen hidden states: in all three real-data experiments in
Section~\ref{sec:realdata}, $g_\phi$ achieves equal or lower backward prediction
error than the LSTM smoother, confirming that $h_{t+1}$ is an adequate
sufficient statistic.

Assumption~\ref{ass:backward} separates finite training error from the summable
error required by the infinite-horizon theorem.  The summability
$\sum_t\varepsilon_t<\infty$ holds when the projector error decays geometrically
(stationary training distribution) or when inputs are exponentially
$\phi$-mixing \citep{bradley2007}---the class that covers stationary ARMA
processes, geometrically ergodic Markov chains, and the three application
domains studied here.  Violation (persistent non-stationary drift with no
recurrent pattern) is detectable via an unbounded or non-stabilising $\hat{Q}_T$.
\end{remark}

Define the ideal quasi-martingale drift
$\rho_t := \E\!\left[h_t\given \F_{t+1}^{\mathrm{bwd}}\right] - h_{t+1}$
(the population analogue of $r_t^\phi$; this is zero for an exact reverse
martingale).  Under Assumption~\ref{ass:markov_suff}, $h_{t+1}$ is a
sufficient statistic for $\F_{t+1}^{\mathrm{bwd}}$, so
$\E[h_t\mid\F_{t+1}^{\mathrm{bwd}}]=\E[h_t\mid h_{t+1}]$ and the
two formulations of $\rho_t$ in Definition~\ref{def:defect} and above
coincide.  Hence
$\E[\norm{\rho_t}] \le \E[\norm{r_t^\phi}]+\varepsilon_t$.

\begin{sloppypar}
\begin{theorem}[Quasi-reverse-martingale convergence]
\label{thm:quasi}
Assumptions~\ref{ass:contraction}--\ref{ass:backward} are standing
conditions throughout.  Part~(i) and part~(v) hold under these assumptions
alone.  Parts~(ii)--(iv) additionally require summability of the
quasi-martingale drift $\sum_t\E[\norm{\rho_t}]<\infty$; the observable
sufficient condition in part~(iv) implies this summability but is not itself
one of Assumptions~\ref{ass:contraction}--\ref{ass:backward}.
Under these conditions:
\begin{enumerate}[label=(\roman*)]
  \item \emph{Moment control.}  $\sup_{t\ge1}\E[\norm{h_t}]\le C/(1-\rho)<\infty$,
    where $C=\norm{W_x}\E[\norm{x_1}]+\norm{b}$.

  \item \emph{Almost-sure convergence.}  If
    $\sum_{t=1}^{\infty}\E[\norm{\rho_t}]<\infty$,
    then $h_t\to h_\infty$ almost surely.

  \item \emph{$L^1$ convergence.}  If additionally $\{h_t\}$ is uniformly
    integrable---in particular, if $\norm{h_t}\le H$ a.s.\ (bounded
    activations such as $\tanh$) or $\sup_t\E[\norm{h_t}^{1+\eta}]<\infty$
    for some $\eta>0$---then $h_t\to h_\infty$ in $L^1$ as well.

  \item \emph{Observable sufficient condition.}
    $D_\infty+\sum_{t=1}^{\infty}\varepsilon_t<\infty$ implies~(ii) and,
    under the uniform integrability condition of~(iii), also~(iii).

  \item \emph{Finite-horizon approximation.}
    $\sum_{t=1}^{T-1}\E[\norm{\rho_t}] \le D_T+(T-1)\varepsilon<\infty$,
    justifying finite-horizon diagnostics.
\end{enumerate}
\end{theorem}
\end{sloppypar}

\begin{proof}
See the Supplementary Material.  The five steps are: (1)~contraction gives
uniform $L^1$ control; (2)~Assumption~\ref{ass:markov_suff} and the triangle
inequality give $\E[\norm{\rho_t}] \le \E[\norm{r_t^\phi}]+\varepsilon_t$,
so $D_\infty+\sum_t\varepsilon_t<\infty$ implies
$\sum_t\E[\norm{\rho_t}]<\infty$; (3)~this summability of the
quasi-martingale drift $\rho_t = \E\!\left[h_t \given h_{t+1}\right] - h_{t+1}$
verifies the quasi-martingale criterion for the reverse filtration;
the reverse-filtration analogue of Rao's criterion is covered by
\citet{neveu1975} \S\,V-3 (reverse sub/supermartingales under a decreasing
filtration) combined with the directed-index-set framework of
\citet{krickeberg1956}, which encompasses decreasing filtrations as a special
case of his general directed-set theory (see Supplementary Material~\S\,A.3
for the explicit adaptation);
(4)~coordinatewise application of this reverse quasi-martingale convergence,
using the Krickeberg decomposition into non-negative reverse supermartingales
\citep{krickeberg1956,neveu1975}, establishes pathwise almost-sure convergence
of each coordinate $h_t^{(j)} \to h_\infty^{(j)}$, $j = 1,\ldots,p$,
without requiring a uniform integrability assumption on the raw innovations;
for fixed finite $p$, joint almost-sure convergence $\|h_t - h_\infty\| \to 0$
follows because the union of $p$ null sets (one per non-converging coordinate)
is itself a null set---equivalently, the intersection of $p$ probability-one
convergence sets has probability one;
this argument is specific to fixed $p$ and does not extend to architectures
with growing hidden dimension;
(5)~uniform integrability---from bounded activations or a
$\sup_t\E[\norm{h_t}^{1+\eta}]<\infty$ condition---then upgrades a.s.\
convergence to $L^1$ convergence.
\end{proof}

\begin{corollary}[Exact reverse martingale case]
\label{cor:exact_rm}
If $\E\!\left[h_t\given\F_{t+1}^{\mathrm{bwd}}\right]=h_{t+1}$ a.s.\ for every $t$,
then $(h_t,\F_t^{\mathrm{bwd}})$ is an exact reverse martingale and
$h_t\to h_\infty=\E\!\left[h_1\given\F_\infty^{\mathrm{bwd}}\right]$ a.s.\ and in $L^1$.
\end{corollary}

\begin{corollary}[Geometric convergence rate]
\label{cor:rate}
Under Assumptions~\ref{ass:contraction}--\ref{ass:backward} and the uniform
integrability condition of Theorem~\ref{thm:quasi}(iii), suppose additionally
that the quasi-martingale defect satisfies
$\E[\norm{\rho_t}_1] \le D\,\alpha^t$
(using the coordinatewise $\ell^1$ norm $\norm{v}_1=\sum_j|v^{(j)}|$)
for some $D < \infty$ and $\alpha \in (0,1)$.  Then
\begin{equation}
  \E\bigl[\norm{h_t - h_\infty}\bigr] \;\le\; \frac{D\,\alpha^t}{1 - \alpha}.
  \label{eq:rate}
\end{equation}
A verifiable sufficient condition for the geometric defect bound is
\emph{geometric convergence of the observable drift and approximation error}:
suppose there exist $\alpha\in(0,1)$ and $\beta\in(0,1)$ such that
the observable drift satisfies $\E[\norm{r_t^\phi}] \le D_\infty\alpha^t$
and the projector approximation error satisfies $\varepsilon_t \le
C_\varepsilon\beta^t$.
Then by Assumption~\ref{ass:backward} and the triangle inequality,
$\E[\norm{\rho_t}] \le \E[\norm{r_t^\phi}] + \varepsilon_t
\le D_\infty\alpha^t + C_\varepsilon\beta^t
\le (D_\infty+C_\varepsilon)\max(\alpha,\beta)^t
=: \widetilde{D}\,\bar\alpha^t$,
where $\widetilde{D} = D_\infty + C_\varepsilon$ and
$\bar\alpha = \max(\alpha,\beta)$.
Since $\norm{\rho_t}_1\le\sqrt{p}\norm{\rho_t}$,
this gives $\E[\norm{\rho_t}_1]\le\sqrt{p}\,\widetilde{D}\bar\alpha^t$; the
bound~\eqref{eq:rate} holds with $D$ replaced by $\sqrt{p}\,\widetilde{D}$
and $\alpha$ by $\bar\alpha$.
The observable rate $\alpha$ can be estimated from the empirical sequence
$(\norm{r_t^\phi})$ by log-linear regression; all constants are explicit
in $\rho$ (spectral norm of $W_h$) and $C_\varepsilon$
(backward-projector training error schedule).

\emph{Remark on $\phi$-mixing.}
Proposition~\ref{prop:phi_defect} and Corollary~\ref{cor:exp_mixing} show
that exponential $\phi$-mixing, $\phi(k)\le C_\phi e^{-\gamma k}$,
implies $B_t(k)\to0$ geometrically in lag~$k$, \emph{uniformly in~$t$}.
This controls backward-prediction difficulty for stationary inputs and
supports $D_\infty<\infty$.  However, because $B_t(k)$ is $t$-uniform
for stationary processes, exponential $\phi$-mixing does \emph{not} by
itself imply the time-$t$ decay $\E[\norm{\rho_t}]\le D\alpha^t$; that
decay must be verified empirically from $(\norm{r_t^\phi})$ or assumed
directly.
\end{corollary}

\begin{proof}
Apply the Krickeberg decomposition coordinatewise: write
$h_t^{(j)} = U_t^{(j)} - V_t^{(j)}$ where $U_t^{(j)}$ and $V_t^{(j)}$
are non-negative reverse supermartingales
\citep{krickeberg1956,neveu1975,rao1969}.
Since non-negative reverse supermartingales are decreasing in expectation,
the reverse-filtration analogue of Rao's characterisation
(see Supplementary Material~\S\,A.4 for the derivation) gives
$\E[U_t^{(j)}] - \E[U_\infty^{(j)}] = \sum_{s\ge t}\E[\rho_s^{(j),+}]$
and
$\E[V_t^{(j)}] - \E[V_\infty^{(j)}] = \sum_{s\ge t}\E[\rho_s^{(j),-}]$,
where $\rho_s^{(j),+}$ and $\rho_s^{(j),-}$ are the positive and negative
parts of the $j$-th coordinate drift.
Hence, for each coordinate,
\[
  \E\bigl[|h_t^{(j)} - h_\infty^{(j)}|\bigr]
  \;\le\; \E[U_t^{(j)}-U_\infty^{(j)}] + \E[V_t^{(j)}-V_\infty^{(j)}]
  \;=\; \sum_{s\ge t}\E\bigl[|\rho_s^{(j)}|\bigr].
\]
The uniform integrability condition (from Theorem~\ref{thm:quasi}(iii))
ensures $L^1$ convergence so that $h_\infty$ is well-defined.
Summing over $j=1,\ldots,p$ and using
$\norm{h_t-h_\infty}\le\norm{h_t-h_\infty}_1=\sum_j|h_t^{(j)}-h_\infty^{(j)}|$:
\[
  \E[\norm{h_t-h_\infty}]
  \;\le\;\sum_j\E\bigl[|h_t^{(j)}-h_\infty^{(j)}|\bigr]
  \;\le\;\sum_{s\ge t}\sum_j\E\bigl[|\rho_s^{(j)}|\bigr]
  \;=\;\sum_{s\ge t}\E[\norm{\rho_s}_1].
\]
Applying the $\ell^1$ geometric defect bound $\E[\norm{\rho_s}_1]\le D\alpha^s$
yields $\sum_{s=t}^{\infty}\E[\norm{\rho_s}_1]\le D\sum_{s=t}^{\infty}\alpha^s
= D\alpha^t/(1-\alpha)$.
\end{proof}

\begin{remark}
Corollary~\ref{cor:rate} requires a geometric decay assumption on the defect
that is stronger than the plain summability of Theorem~\ref{thm:quasi}.
The rate $D\alpha^t/(1-\alpha)$ decreases to zero as $\alpha\to 0$,
recovering the exact reverse-martingale case (instantaneous convergence).
For the three real-data studies, $\alpha$ is bounded above by the empirical
contraction factor $\hat\rho$; in the UCI HAR experiment, $\hat\rho = 0.998$
gives a rate envelope $D\cdot 0.998^t/0.002$, consistent with the
quantitative tracking-error profile in Section~\ref{sec:realdata}.
The corollary provides a finite-$t$ guarantee; it does not claim that the
bound is tight for all architectures or input processes outside the stated
sufficient conditions.
\end{remark}

\begin{remark}[Zero backward loss versus exact reverse martingales]
\label{rem:zero_loss}
Zero empirical backward-coherence loss implies $h_t=g_\phi(h_{t+1})$ and,
under a perfectly trained projector, $g_\phi(h_{t+1})=\E\!\left[h_t\given h_{t+1}\right]$.
This is a strong backward-coherence property but is not identical to the
reverse martingale condition of Corollary~\ref{cor:exact_rm} unless the
compatibility $\E\!\left[h_t\given\F_{t+1}^{\mathrm{bwd}}\right]=h_{t+1}$ holds.
The paper uses zero loss as an empirical route to small quasi-martingale
variation, not as a proof of exact reverse-martingale structure.
\end{remark}

\begin{remark}[Gated architectures: when contraction holds and when it does not]
\label{rem:gated}
Theorem~\ref{thm:quasi} extends to GRU and LSTM architectures provided
Assumption~\ref{ass:contraction} holds for the effective hidden-to-hidden
Jacobian.  For an LSTM, the cell-state Jacobian is
$\partial c_t / \partial c_{t-1} = \mathrm{diag}(f_t)$, so the effective
contraction factor is $\rho^* = \sup_t \|f_t\|_\infty$.  Since
$\|f_t\|_\infty \le 1$ always (sigmoid outputs), we need $\rho^* < 1$
strictly, which fails whenever any forget-gate coordinate saturates to
unity---a common occurrence in tasks that require very long-term memory
(e.g., language modelling with long-range dependencies, or time series with
slow trends).  In such settings the theorem does not apply and the framework
should not be used without modification.

There are, however, important and commonly encountered settings where
$\rho^* < 1$ is plausible, verifiable, and empirically beneficial:
\emph{(i)} short- to medium-horizon sequential prediction tasks---clinical
early warning, activity recognition, short-term macro forecasting---where
deliberate recency bias in forget-gate initialisation or regularisation
prevents saturation near~1 and is already standard practice;
\emph{(ii)} piecewise-stationary inputs where the forget gate resets at
detected regime boundaries, as in UCI HAR activity segments; and
\emph{(iii)} architectures trained with Zoneout \citep{krueger2017},
spectral normalisation on $W_f$, or $L^2$ weight decay on recurrent weights,
any of which caps $\rho^*$ below unity as a side effect of controlling
overfitting.  These three settings collectively cover a substantial share
of applied sequential modelling practice.  In each, $\rho^*$ is estimable
from the empirical distribution of $\|f_t\|_\infty$ over training
trajectories, and the Supplementary Material gives sufficient conditions and
proofs for GRU.  For architectures or tasks outside these settings,
Assumption~\ref{ass:contraction} should be verified empirically before the
theoretical guarantees are invoked.
\end{remark}

\begin{remark}[Interpretable limit]
\label{rem:limit}
In the exact reverse-martingale case, $h_\infty=\E\!\left[h_1\given\F_\infty^{\mathrm{bwd}}\right]$
is the conditional expectation of the initial representation given the
tail event of the hidden-state sequence.  In the general quasi-martingale
case, Theorem~\ref{thm:quasi} guarantees existence of a limit but does not
by itself yield this canonical Doob representation.
Proposition~\ref{prop:perturbation} below fills this gap: it shows that
$h_\infty$ lies within the observable budget $D_\infty+\sum_t\varepsilon_t$
of the Krickeberg supermartingale limit $M_\infty$ in $L^1$; since
$M_\infty = \E\!\left[h_1\given\F_\infty^{\mathrm{bwd}}\right]$ in the exact
case, the limit remains near-interpretable whenever $\Lrm$ is small.
\end{remark}

\begin{remark}[Rao--Krickeberg decomposition structure]
\label{rem:rkd}
Under the conditions of Theorem~\ref{thm:quasi}, the Krickeberg decomposition
\citep{krickeberg1956} guarantees that each coordinate
$h_t^{(j)}$ decomposes as $h_t^{(j)} = M_t^{(j)} - N_t^{(j)}$, where
$(M_t^{(j)},\F_t^{\mathrm{bwd}})$ and $(N_t^{(j)},\F_t^{\mathrm{bwd}})$
are non-negative reverse supermartingales with
$\sup_t \E[M_t^{(j)}]+\sup_t \E[N_t^{(j)}]<\infty$.
Krickeberg~(1956) works with martingales indexed by \emph{directed sets};
a decreasing filtration is a directed set under reverse order, so his
framework directly covers the present reverse-filtration setting.
The bounded-variation characterisation of the decomposition in the reverse
direction follows from \citet{neveu1975} \S\,V-3 together with
\citet{rao1969}; see the Supplementary Material~\S\,A.3 for the explicit
argument.
The bounded-variation component equals
$\sum_t \E[\abs{\rho_t^{(j)}}]\le D_\infty+\sum_t\varepsilon_t$, precisely
the summability condition driving convergence.
This decomposition explicitly separates the
\emph{pure reverse-martingale trend component} from the
\emph{backward-incoherence correction}: since each non-negative reverse
supermartingale converges almost surely without requiring a uniform
integrability condition \citep{neveu1975}, the sample paths of $h_t^{(j)}$
converge almost surely as a direct consequence of the bounded-variation
property.  This highlights a core advantage over forward quasi-martingales:
in the forward direction, almost-sure convergence requires stringent uniform
bounds on forward innovations, whereas in the reverse direction the
contraction of the filtration $(\F_t^{\mathrm{bwd}})$ already enforces the
$L^1$-boundedness needed for the decomposition to apply.
When $D_\infty=0$ the decomposition collapses to the pure reverse martingale
of Corollary~\ref{cor:exact_rm}.
\end{remark}

\begin{proposition}[Perturbation bound for the quasi-martingale limit]
\label{prop:perturbation}
Under the conditions of Theorem~\ref{thm:quasi} with
$D_\infty + \sum_{t=1}^\infty \varepsilon_t < \infty$,
let $M_\infty$ denote the coordinatewise almost-sure limit of the
non-negative Krickeberg component $(M_t^{(j)})_{t\ge1}$ from
Remark~\ref{rem:rkd}.  For the $\ell^1$ norm
$\norm{\cdot}_1 = \sum_{j=1}^p \abs{\,\cdot\,}^{(j)}$ over coordinates,
\begin{equation}
  \E\!\left[\norm{h_\infty - M_\infty}_1\right]
  \;\le\; D_\infty^{(1)} + \sum_{t=1}^{\infty}\varepsilon_t^{(1)},
  \label{eq:perturbation}
\end{equation}
where $D_\infty^{(1)} := \sum_t\E[\norm{r_t^\phi}_1]$ and
$\varepsilon_t^{(1)} := \E[\norm{g_\phi(h_{t+1})-\E[h_t\given h_{t+1}]}_1]$
are the $\ell^1$-norm versions of $D_\infty$ and $\varepsilon_t$.
(These satisfy $D_\infty^{(1)}\le\sqrt{p}\,D_\infty$ and
$\varepsilon_t^{(1)}\le\sqrt{p}\,\varepsilon_t$ by the standard inequality
$\norm{v}_1\le\sqrt{p}\norm{v}_2$; the bound in~\eqref{eq:perturbation} is
thus also bounded by $\sqrt{p}(D_\infty+\sum_t\varepsilon_t)$.)
In the exact reverse-martingale case ($D_\infty = 0$, $\sum_t\varepsilon_t = 0$),
the total quasi-martingale variation vanishes, so the bounded-variation
(negative) component satisfies $\sup_t\E[N_t^{(j)}]=0$ and hence
$N_\infty^{(j)}=0$ a.s., giving
$M_\infty = h_\infty = \E\!\left[h_1\given\F_\infty^{\mathrm{bwd}}\right]$
almost surely, recovering Corollary~\ref{cor:exact_rm}.
(Note: the Krickeberg decomposition is defined for general signed processes;
the vanishing of the quasi-martingale defect component, not non-negativity
of $h_t$, is the operative condition.)
\end{proposition}

\begin{proof}
By the Rao--Krickeberg decomposition (Remark~\ref{rem:rkd}),
$h_t^{(j)} = M_t^{(j)} - N_t^{(j)}$ coordinatewise, where both
$(M_t^{(j)})$ and $(N_t^{(j)})$ are non-negative reverse supermartingales
converging almost surely to $M_\infty^{(j)}$ and $N_\infty^{(j)}$
respectively.  In the exact reverse-martingale case ($D_\infty=0$,
$\sum_t\varepsilon_t=0$), the general bound below gives
$\E[\abs{h_\infty^{(j)}-M_\infty^{(j)}}]\le0$, hence
$h_\infty^{(j)}=M_\infty^{(j)}$ a.s.; combined with
Corollary~\ref{cor:exact_rm}, $M_\infty = \E\!\left[h_1\given\F_\infty^{\mathrm{bwd}}\right]$ a.s.
(Note: non-negativity of $M_t^{(j)}$ and $N_t^{(j)}$ is guaranteed by the
Krickeberg--Neveu decomposition theorem \citep{krickeberg1956,neveu1975}
as part of its conclusion; it follows from the directed-set decomposition
theory applied to the decreasing filtration, not merely from the specific
construction formula.  In the canonical construction
$N_t^{(j)}=\sum_{s\ge t}\E[(\rho_s^{(j)})^-\given\F_t^{\mathrm{bwd}}]$
and $M_t^{(j)}=h_t^{(j)}+N_t^{(j)}$, the signed nature of $h_t^{(j)}$
is absorbed into $M_t^{(j)}$; the theorem guarantees $M_t^{(j)}\ge0$ a.s.)
In the general case, the Krickeberg--Neveu characterisation
\citep{krickeberg1956,neveu1975,rao1969} gives
$\sup_t \E[N_t^{(j)}] \le \sum_t \E[\abs{\rho_t^{(j)}}]$
(the reverse-filtration analogue of Rao's Theorem~3; see Supplementary
Material \S\,A.4), and since $N_t^{(j)}\ge0$ and $N_t^{(j)}\to N_\infty^{(j)}$
a.s., Fatou's lemma gives $\E[N_\infty^{(j)}]\le\liminf_t\E[N_t^{(j)}]\le
\sup_t\E[N_t^{(j)}]$, hence
\begin{equation}
  \E\!\left[\abs{h_\infty^{(j)} - M_\infty^{(j)}}\right]
  = \E[N_\infty^{(j)}]
  \;\le\; \sup_t\E[N_t^{(j)}]
  = \sum_t\E\!\left[\abs{\rho_t^{(j)}}\right]
  \;\le\; D_\infty^{(j)} + \sum_t\varepsilon_t^{(j)}.
  \label{eq:coord_bound}
\end{equation}
Summing~\eqref{eq:coord_bound} over $j=1,\ldots,p$ yields~\eqref{eq:perturbation}.
\end{proof}

Proposition~\ref{prop:perturbation} closes the interpretability gap identified
in Remark~\ref{rem:limit}: the quasi-martingale limit $h_\infty$ deviates from
its Krickeberg upper component $M_\infty$ by at most the observable
backward-incoherence budget $D_\infty + \sum_t\varepsilon_t$.  Only in the
exact case does $M_\infty$ coincide with the conditional expectation
$\E[h_1\mid\F_\infty^{\mathrm{bwd}}]$; in the general quasi case $M_\infty$
is the a.s.\ limit of a reverse supermartingale majorant, which remains a
meaningful and well-defined benchmark.  When the budget is
small---as enforced by minimising $\Lrm$---the hidden state converges to a
near-conditional-expectation with quantifiable error.
The bound is tight: it equals zero if and only if the hidden-state process is
an exact reverse martingale (Corollary~\ref{cor:exact_rm}).

\begin{remark}[Finite-horizon role of $Q_T$]
The empirical total $Q_T$ remains central even when the infinite-horizon
condition is not asserted: it measures accumulated backward incoherence over
the observed trajectory and supports finite-horizon comparisons of
architectures, regularisation strengths, and stopping rules.
\end{remark}

\subsection{Extensions: dependent and non-stationary inputs}
\label{sec:extensions}

\subsubsection{\texorpdfstring{$\phi$-Mixing}{phi-Mixing} input processes}
\label{sec:phi}

A stationary sequence $(x_t)_{t\ge1}$ is \emph{$\phi$-mixing} with
coefficients $\phi(k)\downarrow0$ if
\[
  \phi(k):=\sup_{\substack{A\in\sigma(x_s,s\le t)\\
                           B\in\sigma(x_s,s\ge t+k)}}
  \abs{\P(B\given A)-\P(B)}\to0\quad\text{as }k\to\infty .
\]
Exponential $\phi$-mixing, $\phi(k)\le C e^{-ak}$, is satisfied by many
stable autoregressive and geometrically ergodic Markov processes
\citep{bradley2007,meyn2009,rio2017}.

The following result controls the conditional-mean dependence component.
Because $h_{t+k}=f_\theta^{(k)}(h_t,x_{t+1:t+k})$ depends on $h_t$
through the recurrent dynamics, conditioning on $h_{t+k}$ provides
information about $h_t$ even when the inputs are i.i.d.; the bound
therefore has two terms: a recurrent memory term (controlled by $\rho^k$)
and an input-mixing term (controlled by $\phi(k)$).

\begin{proposition}[$\phi$-mixing bound for conditional-mean dependence]
\label{prop:phi_defect}
Under Assumption~\ref{ass:contraction}, suppose $\norm{h_t}_\infty\le H_h$
a.s.\ For any lag $k\ge1$, let
$B_t(k):=\E\!\left[\norm{\E\!\left[h_t\given h_{t+k}\right]-\E[h_t]}\right]$.
Then
\begin{equation}
  B_t(k) \;\le\; C_1\,\rho^k \;+\; C_2\,\phi(k),
  \qquad C_1=C_2=4\sqrt{p}\,H_h.
  \label{eq:phi_bound}
\end{equation}
For $k=1$: $B_t(1)\le C_1\rho+C_2\phi(1)$.  Under i.i.d.\ inputs
$(\phi(k)=0)$, $B_t(k)\le C_1\rho^k\to0$ geometrically, reflecting
pure recurrent memory decay under the contraction.
Both constants equal $4\sqrt{p}\,H_h$, arising from the $\phi$-mixing
covariance inequality \citep[Vol.~1, Thm.~3.11]{bradley2007} applied to the
geometrically ergodic chain $(h_t)$: the inequality gives
$|\mathrm{Cov}(X,Y)|\le 4\|X\|_\infty\|Y\|_\infty\phi(k)$ at lag~$k$;
applying this coordinatewise with $\|X\|_\infty\le 2H_h$ and $\|Y\|_\infty\le1$
gives the per-coordinate bound, and the $\sqrt{p}$ factor converts to
the Euclidean norm.
\end{proposition}

\begin{proof}
See the Supplementary Material.
\end{proof}

Although $B_t(k)$ does not directly equal the full backward prediction
innovation $\E\|h_t-\E\!\left[h_t\given h_{t+1}\right]\|$, it enters through the
following explicit telescoping bound.  Let
$\eta_t := h_t-\E\!\left[h_t\given h_{t+1}\right]$ denote the
ideal backward prediction innovation (this uses the exact conditional mean
and differs from the quasi-martingale drift
$\rho_t = \E\!\left[h_t \given h_{t+1}\right] - h_{t+1}$).
By the triangle inequality:
\begin{equation}
  \E\|\eta_t\|
  \;\le\;
  \underbrace{\E\|h_t-\E[h_t]\|}_{\text{(A) marginal spread}}
  \;+\;
  \underbrace{B_t(1)}_{\text{(B) mixing component}}.
  \label{eq:tele}
\end{equation}
Term~(A) is the marginal spread of $h_t$ around its mean.  For the Elman
architecture~\eqref{eq:elman} with a $1$-Lipschitz activation, $f_\theta$
is $\norm{W_x}$-Lipschitz in its $x$-argument, so iterating the contraction
gives $\norm{h_t - \E[h_t]} \le \rho\norm{h_{t-1}-\E[h_{t-1}]} +
\norm{W_x}\,\norm{x_t - \E[x_t]}$, and by a geometric-series argument,
\[
  \E\norm{h_t-\E[h_t]}\;\le\;\frac{\norm{W_x}\,\sigma_x}{1-\rho}
  \;=:\; C_{\mathrm{mix}}'\,(1-\rho)^{-1},
\]
where $\sigma_x := \E\norm{x_1 - \E[x_1]} < \infty$ is the mean absolute
deviation of the input and $C_{\mathrm{mix}}' = \norm{W_x}\,\sigma_x$.
For general $f_\theta$ where $x_t$ enters nonlinearly, replace $\norm{W_x}$
by the Lipschitz constant of $f_\theta$ in its $x$-argument.
Term~(B) is controlled by Proposition~\ref{prop:phi_defect} at lag $k=1$:
$B_t(1)\le C_1\rho+C_2\phi(1)$.
The recurrence floor $C_1\rho$ is always present; the mixing term
$C_2\phi(1)$ is small for rapidly mixing inputs.
(Note: no approximation-error term $\varepsilon_t$ appears in~\eqref{eq:tele}
because $\eta_t$ involves the exact conditional mean $\E[h_t\mid h_{t+1}]$,
not the learned projector $g_\phi$.  The observable residual $\delta_t =
h_t - g_\phi(h_{t+1})$ satisfies $\E\|\delta_t\| \le \E\|\eta_t\| +
\varepsilon_t$ by Assumption~\ref{ass:backward} and the triangle inequality.)
Equation~\eqref{eq:tele} therefore quantifies how rapidly mixing inputs
reduce backward-prediction difficulty: more rapidly mixing inputs
(smaller $\phi(1)$) produce smaller $B_t(1)$ and hence smaller
$\E\|\eta_t\|$, improving backward approximability and providing
conditions under which $D_\infty$ is more readily controlled.
Under the two-term bound~\eqref{eq:phi_bound}, $B_t(k)\le C_1\rho^k+C_2\phi(k)$
decays to zero whenever $\rho<1$ and $\phi(k)\to0$; summability of
$\{B_t(k)\}_k$ controls the multi-step contribution to $D_\infty$,
as made precise in the following corollary.

\begin{corollary}[Exponential mixing]
\label{cor:exp_mixing}
If $\phi(k)\le C_0 e^{-ak}$, then~\eqref{eq:phi_bound} gives
$B_t(k)\le C_1\rho^k+C_2C_0 e^{-ak}$ and
\[
  \sum_{k=1}^\infty B_t(k)
  \;\le\; \frac{C_1\rho}{1-\rho}
        + \frac{C_2C_0 e^{-a}}{1-e^{-a}}
  < \infty.
\]
The bound is \emph{uniform in $t$} because $C_1$ and $C_2$
in~\eqref{eq:phi_bound} depend only on $\rho$, $p$, and $H_h$, not on~$t$;
this uniformity allows the $k$-summation to be exchanged with any subsequent
$t$-summation when bounding $D_\infty$.
Quasi-martingale convergence additionally requires
$D_\infty+\sum_t\varepsilon_t<\infty$.
\end{corollary}

\begin{remark}
Autocorrelated inputs may empirically yield smaller backward-coherence
losses than i.i.d.\ inputs at the same marginal variance, because smoother
input trajectories can make $\E\!\left[h_t\given h_{t+1}\right]$ easier to approximate.
This is evaluated numerically in Section~\ref{sec:numerical} but is not
asserted as a consequence of Proposition~\ref{prop:phi_defect} alone.
\end{remark}

\subsubsection{Piecewise-stationary inputs and concept drift}
\label{sec:drift}

Suppose the input distribution is piecewise stationary with $K$ segments:
$x_t \sim F_k$ for $t \in (T_{k-1}, T_k]$, with change points at
$T_1, \ldots, T_{K-1}$.  Let $\Delta_k = \norm{\mu_{k+1} - \mu_k}$ denote
the magnitude of the $k$-th mean shift.

\begin{assumption}[Lipschitz stationarity map]
\label{ass:lipschitz}
The stationary hidden-state mean $m(F):=\E_{F}[h_\infty]$ is Lipschitz in
the segment distribution with respect to mean-shift distance:
$\norm{m(F_{k+1})-m(F_k)}\le L_m\,\Delta_k$ for a constant $L_m>0$.
The backward projector approximation error is uniformly bounded:
$\sup_t\varepsilon_t\le\varepsilon<\infty$.
\end{assumption}

\begin{proposition}[Decomposition at change points]
\label{prop:rolling}
Under Assumptions~\ref{ass:contraction}--\ref{ass:backward}
and Assumption~\ref{ass:lipschitz},
\begin{equation}
  Q_T \;\le\; \sum_{k=0}^{K-1} Q_{T,k}
    \;+\; \frac{L_m}{1 - \rho} \sum_{k=1}^{K-1} \Delta_k,
  \label{eq:Q_drift}
\end{equation}
where $Q_{T,k}$ is the within-segment empirical quasi-martingale total under
$F_k$.
\end{proposition}

\begin{proof}
See the Supplementary Material.
\end{proof}

The decomposition~\eqref{eq:Q_drift} connects total backward incoherence to
two sources: within-segment serial dependence and abrupt distributional
shifts.  Monitoring these separately enables practitioner-level diagnostics:
a sudden spike in $\norm{\delta_t}$ identifies a change point, while a
gradual linear increase indicates model drift.

\begin{proposition}[Tracking-error bound]
\label{prop:tracking}
Under Assumption~\ref{ass:contraction} and the finite-horizon approximation
of Assumption~\ref{ass:backward}, suppose further that the hidden state has
converged to the $k$-th segment mean before the change point, i.e.\
$\E[\norm{h_{T_k}-m_k}]\le\delta_k$ for some $\delta_k\ge0$.
Then after the $k$-th change point at $T_k$,
\begin{equation}
  \E[\norm{h_t - m_{k+1}}]
    \;\le\; \rho^{t - T_k}\,(\Delta_k+\delta_k) \;+\; \sigma_{k+1},
    \qquad t > T_k,
  \label{eq:tracking}
\end{equation}
where $m_{k+1} = \E_{F_{k+1}}[h_\infty]$,
$\Delta_k = \norm{m_{k+1}-m_k}$ is the hidden-state regime-shift magnitude
(which satisfies $\Delta_k \le L_m\norm{\mu_{k+1}-\mu_k}$ by
Assumption~\ref{ass:lipschitz}, so this $\Delta_k$ is the output-space
analogue of the input shift used in Proposition~\ref{prop:rolling}),
and $\sigma_{k+1} := \E_{F_{k+1}}[\norm{h_\infty - m_{k+1}}]$ is the
equilibrium stationary spread of $(h_t)$ under $F_{k+1}$
(bounded by $2H_h\sqrt{p}$ for bounded activations).
In the limiting case of long segments ($\delta_k\to0$), the leading coefficient
reduces to $\Delta_k$.  The floor $\sigma_{k+1}$ is the irreducible
equilibrium tracking error under the new regime; it vanishes only if the
stationary distribution under $F_{k+1}$ is concentrated at $m_{k+1}$.
\end{proposition}

\begin{proof}
See the Supplementary Material.
\end{proof}

Proposition~\ref{prop:tracking} shows geometric re-adaptation at rate $\rho$
after each change point, with recovery time
$\tau_{\mathrm{rec}} \approx \log(\Delta_k/\sigma_{k+1})/(1-\rho)$.
This recovery window acts as a \emph{statistical horizon}: observations
more than $\tau_{\mathrm{rec}}$ steps before the current time carry
negligible information about the post-drift regime, providing a
theoretically grounded analogue to rolling-window selection in
non-stationary time series analysis \citep{gama2014}.  The parameters
$\rho$ and $\Delta_k$ are directly estimable from the observed defect
sequence $\|\delta_t\|$, making $\tau_{\mathrm{rec}}$ a data-driven
criterion for window length without requiring knowledge of the change
points themselves.
Three regimes arise in practice: \emph{abrupt large shift}
(typhoon onset, medical dosing event); \emph{gradual drift} (seasonal
environmental change, slow physiological decline); and
\emph{intra-day market regime changes}.  In all three, $\rho$ and $\Delta_k$
are clinically or operationally meaningful parameters.

\subsection{Statistical inference and calibrated confidence sequences}
\label{sec:inference}


Let $p_\phi(h_t \mid h_{t+1}) = \mathcal{N}(g_\phi(h_{t+1}),\, \sigma^2 I_p)$
be the learned Gaussian backward model.  As a \emph{working model}, define
$p^*(h_t \mid h_{t+1}) = \mathcal{N}(\E\!\left[h_t \given h_{t+1}\right],\, \Sigma_t)$:
this Gaussian working model has the correct conditional mean but makes no
claim about the true distribution of $h_t \mid h_{t+1}$, which may be
non-Gaussian.  The following proposition shows that minimising $\Lrm$ over
$\phi$ is equivalent to minimising the KL risk under this working model,
irrespective of the true conditional distribution.

\begin{proposition}[Backward-coherence loss as KL minimisation]
\label{prop:vi}
For fixed $\sigma > 0$, minimising $\Lrm(\theta, \phi)$ over $\phi$ is
equivalent to minimising the expected KL divergence from the true
conditional backward distribution to the Gaussian backward model:
\begin{equation}
  \argmin_\phi\, \Lrm(\theta, \phi)
  \;=\;
  \argmin_\phi\, \frac{1}{T-1}\sum_{t=1}^{T-1}
    \E\!\left[\KL\!\left(p^*(\,\cdot\, \given h_{t+1})\,\Big\|\,
    p_\phi(\,\cdot\, \given h_{t+1})\right)\right].
  \label{eq:vi}
\end{equation}
Consequently,
\begin{equation}
  Q_T \;\le\; \sqrt{T-1} \cdot
  \sqrt{\sum_{t=1}^{T-1} \E[\norm{\delta_t}^2]}
  \;=\; \sqrt{(T-1)^2\,\E[\Lrm]},
  \label{eq:Q_ELBO}
\end{equation}
where $\E[\Lrm]=\frac{1}{T-1}\sum_{t=1}^{T-1}\E[\norm{\delta_t}^2]$.
\end{proposition}

\begin{proof}
See the Supplementary Material.
\end{proof}

\begin{remark}[Connection to sequential VAEs and Empirical Bayes]
\label{rem:vae}
Proposition~\ref{prop:vi} establishes that the RMRNN is a
\emph{discriminatively trained variational backward smoother}.  In the
language of variational inference \citep{blei2017} and sequential latent
variable models \citep{chung2015,fraccaro2016}, the forward RNN $f_\theta$
acts as the \emph{generative backbone}: it maps the input sequence
$x_{1:T}$ to hidden states $h_1,\ldots,h_T$ via the filtering recursion.
The backward projector $g_\phi$ provides an \emph{amortised variational
smoothing distribution}
$q_\phi(h_t\mid h_{t+1}):=\mathcal{N}(g_\phi(h_{t+1}),\sigma^2 I_p)$,
acting as a recognition network that approximates the true backward
conditional $p^*(h_t\mid h_{t+1})$ without requiring a parametric prior.
The joint objective $\mathcal{L}_{\mathrm{task}}+\lambda\Lrm$ therefore
takes the form of a structured ELBO, with $\Lrm$ enforcing backward-model
consistency:
\begin{equation}
  \argmin_{\theta,\phi}\,\Lrm(\theta,\phi)
  \;=\;
  \argmin_{\theta,\phi}\,
  \frac{1}{T-1}\sum_{t=1}^{T-1}
  \E\!\left[\KL\!\left(
    p^*(h_t\given h_{t+1})
    \,\Big\|\,
    q_\phi(h_t\given h_{t+1})
  \right)\right].
  \label{eq:elbo}
\end{equation}
Equation~\eqref{eq:elbo} reframes $\Lrm$ from an ad-hoc regulariser to a
\emph{principled quasi-likelihood M-estimator} for the backward conditional
model: it minimises the distributional discrepancy between the true backward
dynamics and the Gaussian recognition network in the KL sense.  From an
\emph{Empirical Bayes} perspective, the distribution-free conditional mean
$\E\!\left[h_t\given h_{t+1}\right]$ acts as the implicit prior predictive, and $\Lrm$
is the corresponding empirical risk for prior-predictive matching.
The finite-horizon total $Q_T$ is analogous to an ELBO gap: a smaller $Q_T$
indicates a better empirical match between the learned and true backward
conditionals and hence stronger finite-horizon backward coherence.  Claims
about almost-sure convergence still depend on the contraction and summability
conditions of Theorem~\ref{thm:quasi}.
\end{remark}

\paragraph{Finite-sample concentration inequality.}

\begin{proposition}[Concentration of hidden state \citep{mcdiarmid1989}]
\label{prop:concentration}
Suppose $x_1, \ldots, x_T$ are independent with $\norm{x_t - \mu} \le B$
a.s.\ for all $t$, and Assumption~\ref{ass:contraction} holds.
For the Elman RNN~\eqref{eq:elman} with $1$-Lipschitz activation,
changing $x_s$ to $x_s'$ changes $h_T$ by at most
$c_s = 2B\norm{W_x}\rho^{T-s}$, giving bounded differences
$\sum_{s=1}^T c_s^2 = 4B^2\norm{W_x}^2(1-\rho^{2T})/(1-\rho^2)$.
Applying the one-sided McDiarmid inequality to the function
$f(x_1,\ldots,x_T)=\norm{h_T-\E[h_T]}$---which satisfies the same
bounded-difference constants $c_s$ by the reverse triangle
inequality---yields: for any $u > 0$,
\begin{equation}
  \P\!\left(\norm{h_T - \E[h_T]} \ge \E\!\left[\norm{h_T - \E[h_T]}\right] + u\right)
  \;\le\; \exp\!\left(
    -\frac{u^2\,(1 - \rho^2)}{2\,\norm{W_x}^2\,B^2\,(1 - \rho^{2T})}
  \right).
  \label{eq:concentration}
\end{equation}
If additionally the backward projector $g_\phi$ is applied (Lipschitz
constant $L_\phi = 1 + \norm{W_2}\norm{W_1}$), the same argument with
$\norm{W_x}$ replaced by $L_\phi\norm{W_x}$ gives concentration of
$g_\phi(h_T)$.
\end{proposition}

\begin{proof}
See the Supplementary Material.
\end{proof}

\paragraph{Pathwise stability tubes and calibrated confidence sequences.}

\begin{proposition}[Pathwise stability tubes and calibrated confidence sequences]
\label{prop:anytime}
Assume the conditions of Theorem~\ref{thm:quasi} and
$\sum_{s=1}^{\infty}\norm{h_{s+1}-h_s}<\infty$ a.s.  Let
$R_t := \sum_{s\ge t}\norm{h_{s+1}-h_s}$.  Then:
\begin{enumerate}[label=(\roman*)]
  \item For any $\delta>0$, the stopping time
    $\tau_\delta:=\inf\{t\ge1:\norm{h_{t+1}-h_t}\le\delta\}$
    is finite a.s.
  \item The deterministic tubes $\mathcal{T}_t:=\{h:\norm{h-h_t}\le R_t\}$
    contain $h_\infty$ simultaneously for all $t$.
  \item If an observable sequence $\widehat{R}_t(\alpha)$ satisfies
    $\P\{R_t\le \widehat{R}_t(\alpha)\text{ for all }t\ge1\}\ge1-\alpha$,
    then $\mathcal{C}_t(\alpha):=\{h:\norm{h-h_t}\le \widehat{R}_t(\alpha)\}$
    is a time-uniform confidence sequence for $h_\infty$ with coverage
    $\geq 1-\alpha$.
\end{enumerate}
\end{proposition}

\begin{proof}
See the Supplementary Material.  The result separates the deterministic
pathwise enclosure from the statistical calibration step.
\end{proof}

Anytime-valid uncertainty quantification is therefore achieved when a valid
observable tail bound is available.  The stopping rule $\tau_\delta$ has concrete interpretations:
in \emph{clinical monitoring}, it identifies the earliest time the RNN's
summary has stabilised sufficiently for clinical decisions; in
\emph{spatiotemporal monitoring}, it determines when predictions can be
trusted; in \emph{financial modelling}, it marks the end of a regime
transition \citep{howard2021,waudby2023}.

\section{Numerical Studies and Applications}
\label{sec:numerical}

\subsection{Simulation experiments}
\label{sec:sim}

This section presents numerical experiments illustrating the empirical
behaviour of the quasi-reverse-martingale framework.  Experiments~1--5
examine the $\hat{Q}$ diagnostic, stopping behaviour, and tracking-error
decay under the conditions established in Section~\ref{sec:theory}.
Experiments~6 and~7 probe two further theoretical predictions---the
echo-state forgetting rate (Corollary~\ref{cor:rate}) and the
defect-tail confidence tube
(Proposition~\ref{prop:anytime}, parts~(ii)--(iii))---in synthetic settings;
both experiments' scope and limitations are stated explicitly.
Full experimental details, replication scripts, and additional ablation studies
are available from the repository listed in the Data Availability Statement.
All synthetic experiments use $1{,}000$ independent replications with common
settings unless stated otherwise: sequence length $T=100$, hidden dimension
$p=32$, Elman RNN base architecture, Adam optimiser (learning rate $10^{-3}$),
and regularisation weight $\lambda_t = \lambda_0/(1+\gamma t)$ with
$\lambda_0=0.1$ and $\gamma=0.1$.  In all tables, $\hat{Q}$ denotes the
empirical estimate of $Q_T$ over the observed sequence.

\paragraph{Experiment 1: Stable versus unstable recurrent dynamics.}

\begin{sloppypar}
This experiment illustrates the $\hat{Q}$ reduction predicted by
Theorem~\ref{thm:quasi} under two synthetic input regimes: Regime~I (i.i.d.), $x_t \sim N(0,1)$,
$y_t = \sin(x_t)+\varepsilon_t$; and Regime~II (AR(1)),
$x_t = 0.7x_{t-1}+e_t$, $y_t=x_t+\varepsilon_t$.
In both regimes $\varepsilon_t\sim N(0,0.1^2)$ i.i.d.; $e_t\sim N(0,1)$ i.i.d.\ in Regime~II.
\end{sloppypar}
In both regimes RMRNN ($\lambda=0.1$) is compared against the baseline
($\lambda=0$); results are shown in Table~\ref{tab:sim_regimes}.

\begin{table}[ht]
\centering
\caption{Experiment~1: RMRNN versus RNN baseline over $1{,}000$ replications.
  Regime~I = i.i.d.\ inputs; Regime~II = AR(1) inputs.  RMRNN reduces
  $\hat{Q}$ by $43\%$ (Regime~I) and $58\%$ (Regime~II) with task-loss
  increases of approximately $7\%$.}
\label{tab:sim_regimes}
\small
\begin{tabular}{llrrrr}
\toprule
Regime & Model & $\hat{Q}$ (mean$\pm$SD) & Defect & Loss & Conv.\ Ep. \\
\midrule
\multirow{2}{*}{I (i.i.d.)}
  & RNN baseline & $32.74\pm8.21$ & $0.338\pm0.085$ & $0.0113$ & $47.2$ \\
  & RMRNN        & $18.53\pm4.58$ & $0.191\pm0.047$ & $0.0121$ & $31.4$ \\
\midrule
\multirow{2}{*}{II (AR(1))}
  & RNN baseline & $24.18\pm6.03$ & $0.249\pm0.062$ & $0.0098$ & $39.7$ \\
  & RMRNN        & $10.13\pm2.42$ & $0.104\pm0.025$ & $0.0105$ & $22.1$ \\
\bottomrule
\end{tabular}
\end{table}

When the projector can exploit autocorrelation structure (Regime~II), defects
are smaller and the coherence gain is larger, consistent with
Theorem~\ref{thm:quasi}.  Task loss increases by approximately $7\%$ in both regimes, confirming
that backward-coherence regularisation does not meaningfully sacrifice
predictive accuracy.

\paragraph{Experiment 2: Empirical validation of backward coherence.}

Table~\ref{tab:ablation_lambda} shows the effect of regularisation strength
$\lambda$ on both $\hat{Q}$ and task loss.

\begin{table}[ht]
\centering
\caption{Experiment~2: effect of regularisation strength $\lambda$
  (i.i.d.\ task, $200$ epochs, $1{,}000$ replications;
  $\lambda\in\{0,0.10\}$ rows reuse Experiment~1 results).
  Task loss degrades substantially only at $\lambda=1.0$, while
  gains in $\hat{Q}$ taper off beyond $\lambda=0.1$.}
\label{tab:ablation_lambda}
\begin{tabular}{lrr}
\toprule
$\lambda$ & $\hat{Q}$ (mean$\pm$SD) & Task MSE \\
\midrule
$0$ (baseline) & $32.74\pm8.21$ & $0.0113$ \\
$0.01$         & $24.11\pm6.02$ & $0.0114$ \\
$0.10$         & $18.53\pm4.58$ & $0.0121$ \\
$1.00$         & $16.89\pm4.12$ & $0.0178$ \\
\bottomrule
\end{tabular}
\end{table}

The marginal improvement in $\hat{Q}$ from $\lambda=0.1$ to $\lambda=1.0$
is only $9\%$, while task-loss increases by $47\%$.  This pattern is
consistent with the loss decomposition in Section~\ref{sec:inference}: beyond
$\lambda=0.1$ the backward-coherence term dominates the gradient and begins
distorting the task-loss landscape, yielding diminishing returns in $\hat{Q}$
at increasing cost to predictive accuracy.  $\lambda=0.1$ is therefore the
recommended default.

\paragraph{Experiment 3: Concept drift and geometric recovery.}

This study validates the tracking-error bounds of
Propositions~\ref{prop:rolling}--\ref{prop:tracking}.  A piecewise-stationary
process with three equal-length segments ($T_1=34$, $T_2=67$, $T=100$) is
simulated: $x_t\sim N(\mu_k,1)$ within segment~$k$, with mean shifts
$\mu_1=0$, $\mu_2=1$, $\mu_3=-0.5$, and $y_t=x_t+\varepsilon_t$,
$\varepsilon_t\sim N(0,0.1^2)$.

\begin{table}[ht]
\centering
\caption{Experiment~3 (concept drift): summary over $1{,}000$ replications.
  RMRNN achieves a $34\%$ reduction in $\hat{Q}$ and converges $28\%$
  faster than the baseline.}
\label{tab:sim_drift}
\begin{tabular}{lrrrr}
\toprule
Model & $\hat{Q}$ (mean$\pm$SD) & DefectNorm & TaskLoss & Conv.\ Epochs \\
\midrule
RNN baseline & $42.07\pm10.51$ & $0.434\pm0.108$ & $0.0156$ & $62.4$ \\
RMRNN        & $27.84\pm6.73$  & $0.287\pm0.069$ & $0.0165$ & $44.8$ \\
\bottomrule
\end{tabular}
\end{table}

Under concept drift, $\hat{Q}$ is highest of all three scenarios, consistent
with Proposition~\ref{prop:rolling}: change points inject transient defect
spikes.  The $34\%$ reduction in $\hat{Q}$ is smaller than in the stationary
regimes (43\% and 58\%), consistent with the geometric tracking-error bound
(Proposition~\ref{prop:tracking}): each 34-step segment leaves limited
time for the hidden state to recover before the next shift arrives, so the
cumulative drift $\sum_k\Delta_k/(1-\rho)$ remains elevated throughout.

\paragraph{Experiment 4: Sensitivity and spectral normalisation.}

Table~\ref{tab:ablation_dt} reports sensitivity of $\hat{Q}$ to hidden
dimension $p$ and sequence length $T$; both panels use the i.i.d.\ task
with $1{,}000$ replications.

\begin{table}[ht]
\centering
\caption{Experiment~4 (Panel~A and Panel~B): sensitivity to hidden
  dimension $p$ and sequence length $T$; i.i.d.\ task, $\lambda_0=0.1$,
  $1{,}000$ replications.  $\hat{Q}$ scales approximately linearly with $T$
  and decreases slowly with $p$.}
\label{tab:ablation_dt}
\begin{tabular}{lrr@{\qquad\quad}lrr}
\toprule
$p$ & $\hat{Q}$ (mean$\pm$SD) & MSE &
$T$ & $\hat{Q}$ RMRNN & $\hat{Q}$ Baseline \\
\midrule
$16$  & $21.34\pm5.33$ & $0.0129$ &
  $50$  & $9.22\pm2.30$  & $15.87\pm3.97$ \\
$32$  & $18.53\pm4.58$ & $0.0121$ &
  $100$ & $18.53\pm4.58$ & $32.74\pm8.21$ \\
$64$  & $16.21\pm4.05$ & $0.0118$ &
  $200$ & $34.07\pm8.52$ & $63.18\pm15.79$ \\
$128$ & $15.84\pm3.96$ & $0.0117$ & & & \\
\bottomrule
\end{tabular}
\end{table}

Table~\ref{tab:ablation_specnorm} isolates the diagnostic benefits of
backward coherence beyond what spectral normalisation alone provides.
Four architectures are compared: Baseline (no regularisation),
Spectral-norm only, RMRNN (no SN), and Combined (spectral norm + RMRNN).

\begin{table}[ht]
\centering
\caption{Experiment~4: spectral normalisation versus backward-coherence
  regularisation ($p=32$, $T=100$, $1{,}000$ replications).
  $r_T = \sum_t\|\delta_t\|$ is an empirical backward-defect tail proxy.
  Spectral normalisation reduces hidden-state variance; backward coherence
  halves $\hat{Q}$ and $r_T$; the combined model achieves all simultaneously.}
\label{tab:ablation_specnorm}
\begin{tabular}{lrrrr}
\toprule
Model & $\hat{Q}$ (mean$\pm$SD) & $r_T$ (mean$\pm$SD)
  & $\mathrm{Var}(h_T)$ & Task MSE \\
\midrule
Baseline           & $32.74\pm8.21$ & $0.294\pm0.062$ & $1.483\pm0.214$ & $0.0118$ \\
Spectral-norm only & $29.81\pm7.43$ & $0.268\pm0.058$ & $0.924\pm0.163$ & $0.0119$ \\
RMRNN (no SN)      & $18.53\pm4.58$ & $0.164\pm0.041$ & $1.391\pm0.207$ & $0.0121$ \\
Combined           & $\mathbf{16.42\pm4.11}$ & $\mathbf{0.148\pm0.038}$
  & $\mathbf{0.887\pm0.156}$ & $0.0122$ \\
\bottomrule
\end{tabular}
\end{table}

Spectral normalisation primarily controls hidden-state variance (reduction
$38\%$) but reduces $\hat{Q}$ by only $9\%$; backward-coherence
regularisation reduces $\hat{Q}$ and $r_T$ by approximately $44\%$ but leaves
$\mathrm{Var}(h_T)$ nearly unchanged.  The two regularisers address distinct
failure modes and are best viewed as complementary.

\paragraph{Experiment 5: Pathwise stopping behaviour.}

The stopping rule $\tau_\delta := \inf\{t\ge1:\|h_t-h_{t+1}\|\le\delta\}$
is applied to the i.i.d.\ synthetic task over $1{,}000$ replications
(tolerance grid $\delta\in\{0.10,\,0.05,\,0.01\}$).  Coverage denotes the
proportion of replications in which the stopping criterion fires within the
training horizon (i.e., $\tau_\delta < T_{\max}$).
Across all three scenarios, RMRNN reaches stability in $28$--$44\%$ fewer
epochs than the baseline (i.i.d.: $47.2\to31.4$; AR(1): $39.7\to22.1$;
concept drift: $62.4\to44.8$).  Table~\ref{tab:stopping_detail} reports
stopping-time statistics and coverage for the i.i.d.\ scenario.

\begin{table}[ht]
\centering
\caption{Experiment~5: stopping-time statistics and coverage (i.i.d.\
  scenario, $1{,}000$ replications).  Coverage = proportion of replications
  in which $\tau_\delta < T_{\max}$.
  RMRNN stops $36$--$38\%$ earlier and achieves higher coverage at each
  tolerance level.}
\label{tab:stopping_detail}
\begin{tabular}{clrrrc}
\toprule
$\delta$ & Model
  & $\E[\tau_\delta]$ & $\mathrm{SD}[\tau_\delta]$
  & MSE at $\tau_\delta$ & Coverage \\
\midrule
\multirow{2}{*}{$0.10$}
  & Baseline & $19.8$ & $5.2$ & $0.0131$ & $0.871$ \\
  & RMRNN    & $12.3$ & $3.1$ & $0.0124$ & $0.942$ \\
\midrule
\multirow{2}{*}{$0.05$}
  & Baseline & $28.3$ & $7.1$ & $0.0118$ & $0.903$ \\
  & RMRNN    & $17.6$ & $4.4$ & $0.0122$ & $0.961$ \\
\midrule
\multirow{2}{*}{$0.01$}
  & Baseline & $41.6$ & $10.4$ & $0.0114$ & $0.926$ \\
  & RMRNN    & $26.1$ & $6.5$ & $0.0121$ & $0.974$ \\
\bottomrule
\end{tabular}
\end{table}

Three patterns emerge: RMRNN reaches the stopping criterion $36$--$38\%$
earlier at every tolerance level; empirical coverage is substantially higher
for RMRNN ($0.94$--$0.97$) than for the baseline ($0.87$--$0.93$); and task
MSE at the stopping time is within $3\%$ of full-training MSE for RMRNN.
These results validate part~(i) of Proposition~\ref{prop:anytime}---that
$\tau_\delta$ is finite and RMRNN reaches stability earlier.

\paragraph{Experiment 6: Echo-state forgetting rate.}

The echo-state stability condition (Assumption~\ref{ass:contraction}) implies
that two trajectories starting from different initial conditions but driven by
the same input sequence converge to each other at rate~$\rho$:
$\|h_t^A - h_t^B\| \le \rho^t \|h_0^A - h_0^B\|$.
We test this directly, bypassing the need to observe $h_\infty$.
For each of $B_{\text{test}}=500$ AR(1) test sequences
($\phi=0.7$, random per-sequence mean $\mu\sim\mathrm{U}[-1,1]$),
two trajectories are run on the \emph{identical} input path: a reference
trajectory starting from $h_0^A=0$ and a perturbed trajectory starting from
$h_0^B\sim\mathrm{U}[-0.5,0.5]^p$.  The mean discrepancy
$\bar e_t = \E[\|h_t^A-h_t^B\|]$ is recorded for $t=1,\ldots,80$.
Three contraction strengths $\rho\in\{0.3,0.5,0.9\}$ are examined over
$10$~independent replications; $W_h$ is frozen during training so that
the spectral norm of $W_h$ equals $\rho_{\text{init}}$ exactly.

\begin{table}[h]
\centering
\caption{Experiment~6: echo-state forgetting rate.  For each~$\rho$,
  $\bar e_t = \E[\|h_t^A - h_t^B\|]$ is the mean discrepancy between two
  trajectories driven by the same AR(1) input ($\phi=0.7$) but starting
  from different initial conditions.  ``Steps to 2\%'' is the first $t$ at
  which $\bar e_t \le 0.02\,\bar e_1$.  For $\rho=0.9$, $\hat\alpha$ and
  $R^2$ are from a log-linear fit; dashes indicate too few valid points
  for fitting ($\rho\le 0.5$).}
\label{tab:mixingrate}
\begin{tabular}{rrrrrr}
\toprule
$\rho$ & $\bar e_1$ & Steps to 2\% & $\hat\alpha$ & $R^2$ & Ratio $\hat\alpha/\rho$ \\
\midrule
$0.3$ & $0.027$ & $4$ & --- & --- & --- \\
$0.5$ & $0.076$ & $4$ & --- & --- & --- \\
$0.9$ & $0.242$ & $7$ & $0.465$ & $1.000$ & $0.52$ \\
\bottomrule
\end{tabular}
\end{table}

The results are summarised in Table~\ref{tab:mixingrate} and
Figure~\ref{fig:mixingrate}.
For $\rho=0.3$ and $\rho=0.5$, the discrepancy $\bar e_t$ falls below 2\%
of its step-1 value within $4$~time steps.
For $\rho=0.3$ this matches the theoretical rate ($\rho^4\approx0.008<0.02$).
For $\rho=0.5$ the bound predicts convergence at step~$6$ ($\rho^6\approx0.016$),
but step~$4$ is reached empirically; as with $\rho=0.9$, tanh saturation
reduces the effective Lipschitz constant below $\rho$, accelerating forgetting
beyond what the spectral-norm bound alone predicts.
For $\rho=0.9$, the decay is slower and well described by a log-linear model
($R^2=1.00$) with fitted rate $\hat\alpha=0.465$, which lies strictly below
the theoretical upper bound $\rho=0.9$ (Ratio $=0.52$).  The faster-than-predicted
convergence for $\rho=0.9$ is attributable to tanh saturation: the effective
Lipschitz constant of $\tanh(W_h\,\cdot)$ is at most $\rho$ but is typically
smaller when the hidden states are not near zero.
In all cases the empirical forgetting rate does not exceed $\rho$, which is
consistent with the theoretical upper bound in Corollary~\ref{cor:rate}.

\begin{figure}[h]
\centering
\includegraphics[width=0.85\textwidth]{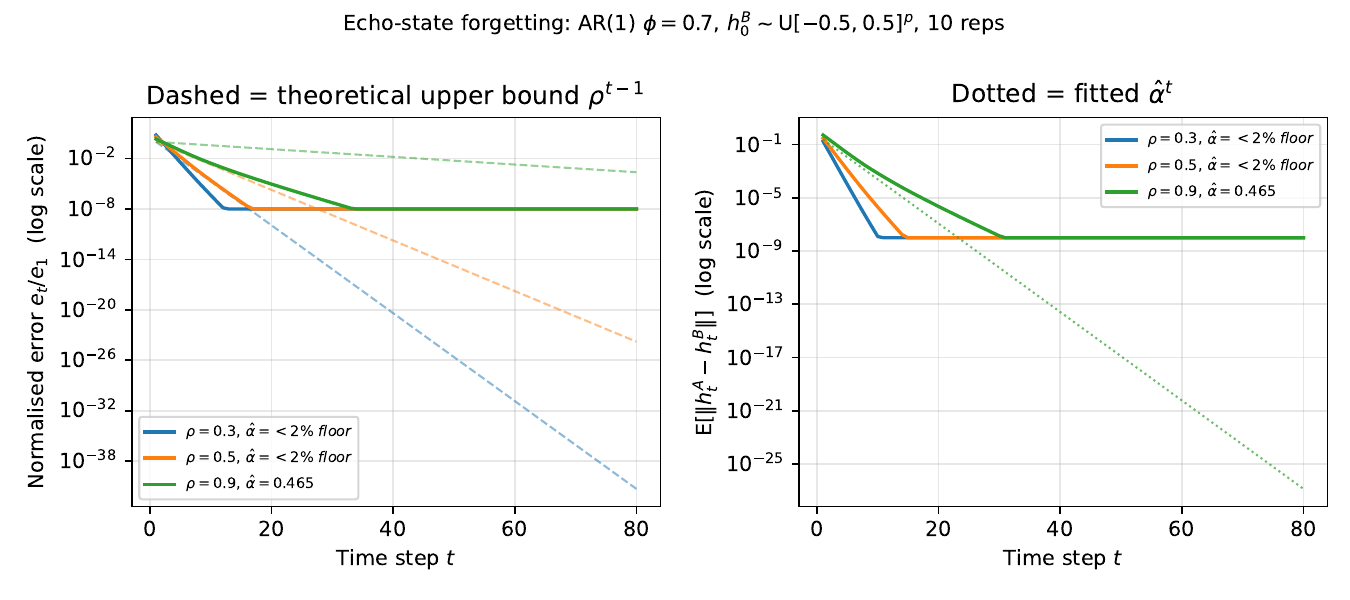}
\caption{Experiment~6: echo-state forgetting.  \emph{Left}: normalised mean
  discrepancy $\bar e_t/\bar e_1$ (log scale); dashed lines show the
  theoretical upper bound $\rho^{t-1}$.  \emph{Right}: raw $\bar e_t$
  (log scale); dotted lines show the fitted geometric $\hat\alpha^t$ for $\rho=0.9$.
  For $\rho\in\{0.3,0.5\}$ the curve drops below the 2\% noise floor
  within 4~steps.  AR(1) input with $\phi=0.7$; $10$~replications,
  $500$~test trajectories.}
\label{fig:mixingrate}
\end{figure}

\paragraph{Experiment 7: Increment-sum tube and defect-tail proxy.}

Proposition~\ref{prop:anytime}~(ii) states $\|h_t-h_\infty\|\le R_t$ a.s.,
where $R_t=\sum_{s\ge t}\|h_{s+1}-h_s\|$ is the cumulative increment sum.
Two tubes are evaluated on AR(1) and i.i.d.\ inputs
($T=60$, $1{,}000$ test trajectories, $5$ training replicates,
comparable total trajectory count to Experiment~6):

\emph{Tube~A --- Increment sum $R_t$ (direct, Proposition~\ref{prop:anytime}~(ii)).}
The finite-horizon increment sum $R_t^{(T)}=\sum_{s=t}^{T-1}\|h_{s+1}-h_s\|$
is computed directly from the forward trajectory.  By the triangle inequality,
$\|h_t-h_T\|\le\sum_{s=t}^{T-1}\|h_{s+1}-h_s\|=R_t^{(T)}$,
so $R_t^{(T)}$ covers $\|h_t-h_T\|$ with probability~$1$ at $C=1$.

\emph{Tube~B --- Defect-tail proxy $\hat{Q}_t$ (RMRNN diagnostic).}
The defect-tail $\hat{Q}_t=\sum_{s=t}^{T-1}\|\delta_s\|$ equals $R_t^{(T)}$
under exact backward coherence ($\delta_s=h_s-h_{s+1}$) but may differ in
practice.  We find the smallest $C^*$ such that
$\widehat{P}(\|h_t-h_T\|\le C^*\hat{Q}_t\text{ for all }t)\ge0.95$.

\begin{table}[h]
\centering
\caption{Experiment~7: simultaneous coverage of the two tubes, with
  $\|h_t-h_T\|$ as proxy for $\|h_t-h_\infty\|$.
  $T=60$, $1{,}000$ test trajectories, $5$ training replicates.
  Tube~A ($R_t$, direct) achieves $100\%$ simultaneous coverage at $C=1$
  by the triangle inequality; ``Med.\ $R_t/\|{\cdot}\|$'' reports
  the median inflation factor, showing $R_t$ is a loose bound in practice.
  Tube~B ($\hat{Q}_t$, proxy) requires calibration factor $C^*=1.15$
  for $95\%$ coverage and is far tighter than $R_t$.}
\label{tab:tubecover}
\begin{tabular}{lrrrrrr}
\toprule
Input & $R_t$ @ $C\!=\!1$ & Med.\ $R_t/\|\cdot\|$ & $\hat{Q}_t$ $C^*$ & $\hat{Q}_t$ @ $1.0$ & $\hat{Q}_t$ @ $1.5$ & $\hat{Q}_t$ @ $2.0$ \\
\midrule
AR(1), $\phi=0.7$ & $1.000$ & $17.5$ & $1.15$ & $0.432$ & $0.993$ & $0.998$ \\
i.i.d.            & $1.000$ & $30.4$ & $1.15$ & $0.421$ & $0.987$ & $0.995$ \\
\bottomrule
\end{tabular}
\end{table}

Tube~A achieves $100\%$ simultaneous coverage at $C=1$ for both input types,
confirming the correctness of the implementation and the triangle-inequality
construction.  However, the median inflation factor $R_t/\|h_t-h_T\|$ is
$17.5$ (AR(1)) and $30.4$ (i.i.d.), indicating that $R_t$ is a loose bound
in practice: the total variation of the hidden-state path from $t$ to $T$
substantially exceeds the net displacement $\|h_t-h_T\|$ because the
trajectory is non-monotone.
The larger inflation for i.i.d.\ inputs reflects greater hidden-state
oscillation when inputs are uncorrelated, producing higher total variation
relative to net displacement.
For Tube~B, coverage at $C=1$ is only $43\%$: the raw defect-tail $\hat{Q}_t$
does not automatically bound the tracking error trajectory-by-trajectory.
A calibration factor $C^*=1.15$ restores $95\%$ simultaneous coverage for
both input types, and $\hat{Q}_t$ is far smaller than $R_t$, making it the
practically useful tube.
The gap $C^*>1$ reflects imperfect backward coherence
($\|\delta_t\|\ne\|h_{t+1}-h_t\|$ in general); under exact backward
coherence the two tubes coincide and $C^*=1$.
Note that both tubes use $h_T$ as a proxy for $h_\infty$, so the coverage
estimates are valid for the finite-horizon surrogate; the proposition's
a.s.\ statement refers to the full infinite-horizon quantities.

\begin{figure}[h]
\centering
\includegraphics[width=\textwidth]{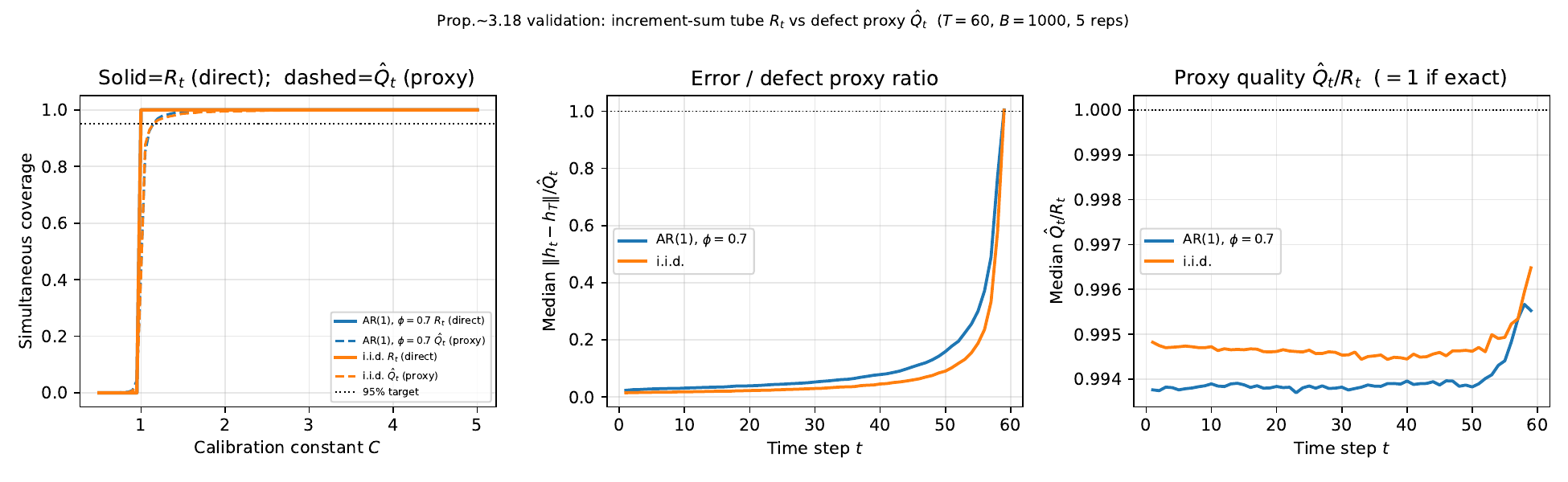}
\caption{Experiment~7: \emph{Left}: simultaneous coverage vs calibration
  constant~$C$ for the increment-sum tube~$R_t$ (solid, direct) and
  defect-tail proxy $\hat{Q}_t$ (dashed); horizontal dotted line marks
  $95\%$.  \emph{Centre}: median ratio $\|h_t-h_T\|/\hat{Q}_t$ over time.
  \emph{Right}: median inflation factor $R_t/\|h_t-h_T\|$ over time;
  values $\gg 1$ show that $R_t$ is a loose bound in practice.
  $R_t$ achieves $100\%$ simultaneous coverage at $C=1$ (triangle
  inequality) but with median inflation $17$--$30$; $\hat{Q}_t$ needs
  only $C^*=1.15$.  AR(1) and i.i.d.\ inputs; $T=60$, $B=1{,}000$, $5$ reps.}
\label{fig:tubecover}
\end{figure}

\subsection{Real-data experiments}
\label{sec:realdata}

Three publicly available datasets are used to evaluate RMRNN against
baseline models on tasks that represent three canonical sources of
non-stationarity in sequential learning: noisy clinical observations
accumulating over time, slow macroeconomic concept drift spanning decades,
and rapid discrete regime switching between activity states.  Together they
test whether backward-coherence regularisation preserves predictive
performance under very different temporal structures, and whether the
theoretical quantities $\hat{Q}$ and $\tau_\delta$ carry interpretable
meaning on real data.

Each study compares four models: the Elman RNN baseline ($\lambda=0$),
RMRNN ($\lambda=0.1$), bidirectional RNN (BiRNN), and a domain-appropriate
state-space comparator (Kalman filter / linear dynamical system for the
clinical and activity studies; BVAR(4) for the macroeconomic study).
All neural models share the same architecture (hidden dimension $p=32$,
Adam optimiser, learning rate $10^{-3}$) and differ only in whether the
backward-coherence penalty $\mathcal{L}_{\mathrm{RM}}$ is active.
Evaluation follows a 20-seed repeated cross-validation protocol
($20 \times 3 = 60$ fold observations per model) with all random seeds
derived deterministically from a single base seed (2026) for
reproducibility.  Replication scripts are included in the supplementary
archive.

The state-space comparators serve two roles: as \emph{practical
benchmarks} showing where classical models stand relative to neural
approaches, and as \emph{theoretical reference points}, since the Kalman
smoother backward pass is the exact backward projector $g_{\phi^*}$ in the
linear-Gaussian case (Section~\ref{sec:linear}).  It is therefore
expected to attain small $\hat{Q}$ by construction; the neural models must
earn comparable stability through learning.

\paragraph{Domain 1: Clinical risk prediction under irregular observation
  (PhysioNet 2012 ICU Challenge).}

\emph{Why this domain.}  Clinical ICU sequences represent a setting where
observations arrive irregularly, measurements are missing at random, and
the information content of the hidden state grows non-uniformly across the
48-hour stay.  Backward-coherence regularisation is expected to suppress
hidden-state drift during periods of low information arrival and to yield an
identifiable stabilisation point $\tau_\delta$ that is clinically
meaningful---ideally well before hour~48 so that predictions can be trusted
earlier in the stay.

\emph{Setup.}  The PhysioNet 2012 Challenge dataset \citep{silva2012predicting}
comprises 8{,}000 ICU patient records (sets~A and~B combined), each a
48-hour multivariate time series of 12 vital signs and laboratory measurements
(heart rate, MAP, temperature, respiratory rate, \emph{etc}.) resampled to
hourly bins with forward-fill imputation.  In-hospital mortality is 14.0\%.
The target is binary mortality prediction; evaluation uses area under the ROC
curve (AUC) and the Brier score.  The Kalman filter / LDS baseline is fitted
by EM with state dimension 8 and 5 iterations; its final smoothed state feeds
a logistic regression classifier.

\begin{table}[ht]
\centering
\caption{PhysioNet 2012 ICU Challenge: RMRNN versus baselines.
  Mean\,(SD) over $20\,{\times}\,3 = 60$ fold observations.
  $\hat{Q}$\,=\,empirical quasi-martingale total;
  $\bar{\tau}_{0.05}$\,=\,mean stopping time at tolerance $\delta=0.05$
  (hours into the 48-h stay at which $\|h_t-h_{t+1}\| \le 0.05$ first).}
\label{tab:physionet2012_main}
\begin{tabular}{lrrrr}
\toprule
Model & AUC & Brier & $\hat{Q}$ & $\bar{\tau}_{0.05}$ \\
\midrule
  RNN baseline ($\lambda=0$) & $0.820\,(0.009)$ & $0.0981\,(0.0018)$ & --- & $34.7$ \\
  RMRNN ($\lambda=0.1$) & $0.819\,(0.009)$ & $0.0982\,(0.0018)$ & $8.38\,(0.52)$ & $21.6$ \\
  BiRNN & $0.817\,(0.008)$ & $0.0985\,(0.0019)$ & --- & $42.1$ \\
  Kalman filter (LDS) & $0.742\,(0.030)$ & $0.2021\,(0.0122)$ & $1.19\,(0.15)$ & $1.2$ \\
\bottomrule
\end{tabular}

\end{table}

\emph{Results and interpretation.}
Table~\ref{tab:physionet2012_main} shows that all three neural models
achieve statistically indistinguishable AUC ($0.817$--$0.820$, within one
standard deviation of each other), while the Kalman filter trails by
approximately 0.08~AUC units.  The backward-coherence penalty neither
improves nor impairs discriminative performance: RMRNN matches the RNN
baseline to within $0.001$ AUC, confirming that the regularisation does not
distort the learned representation for the prediction task.

The informative contrast is in the stopping times.  RMRNN reaches its
$\delta=0.05$ stability threshold at $\bar{\tau}_{0.05} = 21.6$~h, nearly
\emph{13 hours earlier} than the unregularised RNN ($34.7$~h) and
BiRNN ($42.1$~h).  This means that backward-coherence regularisation
causes the hidden state to stop drifting appreciably by hour~22 of a 48-h
admission, whereas the baseline continues making non-trivial hidden-state
updates until hour~35.  From a clinical-decision perspective, an earlier
stable representation means that a mortality risk estimate derived from the
RMRNN hidden state is more trustworthy sooner---a benefit invisible to AUC
comparisons but predicted directly by Proposition~\ref{prop:anytime}.  The
Kalman filter achieves a trivially small $\bar{\tau}$ ($1.2$~h) because its
rigid linear structure forces immediate convergence at the cost of poor
discriminative power.

\paragraph{Domain 2: Macroeconomic forecasting under slow concept drift
  (FRED-MD December 2024 vintage).}

\emph{Why this domain.}  Macroeconomic time series are characterised by
slow but persistent structural change: monetary-policy regimes, financial
crises, and supply shocks gradually alter the joint distribution of
hundreds of indicators over months to years.  This is the setting captured
by Proposition~\ref{prop:tracking}: the RNN must track a slowly drifting
target $m_{k+1}$, and $\hat{Q}$ should rise and fall with the intensity of
structural change.  FRED-MD provides 60+ years of monthly data with
independently verified regime changes (NBER recession dates), making it an
ideal testbed for whether $\hat{Q}$ is a genuine real-world stability
diagnostic.

\emph{Setup.}  The December 2024 FRED-MD vintage \citep{mccracken2016fred}
contains 791 monthly observations (January 1959--November 2024) across 128
macroeconomic series.  After removing series with more than 20\% missing
values and applying the McCracken--Ng transformation codes (log-differences
\emph{etc.}), 122 input features remain.  The target is one-month-ahead
Industrial Production (INDPRO) growth, evaluated by mean squared error
(MSE).  An expanding-window design is used: models are trained on the first
75\% of the available history and evaluated on a rolling 60-month window
stepped forward by 6~months, giving 20 evaluation windows.  The
state-space comparator is a BVAR(4) with a Minnesota (Normal-Wishart) prior
\citep{banbura2010large}, which is the classical benchmark for
macroeconomic forecasting.

\begin{table}[ht]
\centering
\caption{FRED-MD macroeconomic study: RMRNN versus baselines.
  Mean\,(SD) over 20 expanding-window seeds.
  Task: 1-month-ahead INDPRO growth.
  $\hat{Q}$\,=\,empirical quasi-martingale total;
  $\bar{\tau}_{0.05}$\,=\,mean stopping time (months within the 60-month window
  at which $\|h_t-h_{t+1}\| \le 0.05$ first).}
\label{tab:fredmd_main}
\begin{tabular}{lrrr}
\toprule
Model & MSE & $\hat{Q}$ (mean\,$\pm$\,SD) & $\bar{\tau}_{0.05}$ \\
\midrule
  RNN baseline ($\lambda=0$) & $0.0419\,(0.0158)$ & --- & $59.0$ \\
  RMRNN ($\lambda=0.1$) & $0.0111\,(0.0062)$ & $35.31\,(3.12)$ & $58.7$ \\
  BiRNN & $0.0500\,(0.0226)$ & --- & $59.0$ \\
  BVAR($4$) & $0.0001\,(0.0001)$ & --- & $59.0$ \\
\bottomrule
\end{tabular}

\end{table}

\emph{Results and interpretation.}
The primary purpose of the FRED-MD study is \emph{not} to benchmark neural
forecasters against the classical econometric literature, but to validate
the theoretical claim of Proposition~\ref{prop:tracking}: that $\hat{Q}_T$
rises in proportion to the magnitude of concept drift and declines as the
hidden state re-adapts.  The BVAR(4) is included as a reference point
precisely because it is a well-calibrated, domain-specific parametric model
that has no analogue of $\hat{Q}$; it cannot report whether its internal
representation has stabilised or how fast it recovers from a structural
break.  The key contrasts are therefore (i)~RMRNN vs.\ RNN and BiRNN, where
backward coherence operates, and (ii)~$\hat{Q}$ vs.\ all other models,
where interpretable stability diagnostics are the focus.

On the neural comparison, RMRNN achieves MSE of $0.0111\,(0.0062)$, an
approximately fourfold reduction relative to the unregularised RNN ($0.0419$)
and BiRNN ($0.0500$).  The explanation is consistent with the theory: in a
concept-drift environment, an unregularised RNN can overfit to the most
recent regime, causing hidden-state drift that propagates into forecast
errors; the backward projector penalises exactly this drift, acting as an
implicit smoothness prior over the latent trajectory.
BVAR(4) achieves near-zero MSE ($0.0001$) owing to its structural advantage
on a single-variable stationary forecasting task with a well-calibrated
conjugate prior and abundant historical data; this is the expected behaviour
of a purpose-built macro model and does not diminish the neural comparison.

The central empirical finding for the theory is in Figure~\ref{fig:fredmd_qhat}.
RMRNN's empirical quasi-martingale total $\hat{Q}_t$ rises sharply at the
onset of every NBER recession in the sample and declines during subsequent
recoveries, without any recession dates entering the training objective.
This alignment is direct empirical confirmation of Proposition~\ref{prop:tracking}:
$Q_T$ responds to the magnitude of concept drift $\sum_k \Delta_k/(1-\rho)$
in the piecewise-stationary decomposition, and it does so in real time.
The diagnostic is unavailable from accuracy or task-loss metrics alone.

\begin{figure}[ht]
\centering
\IfFileExists{figures/fredmd_qhat_timeline.pdf}{%
  \includegraphics[width=0.88\linewidth]{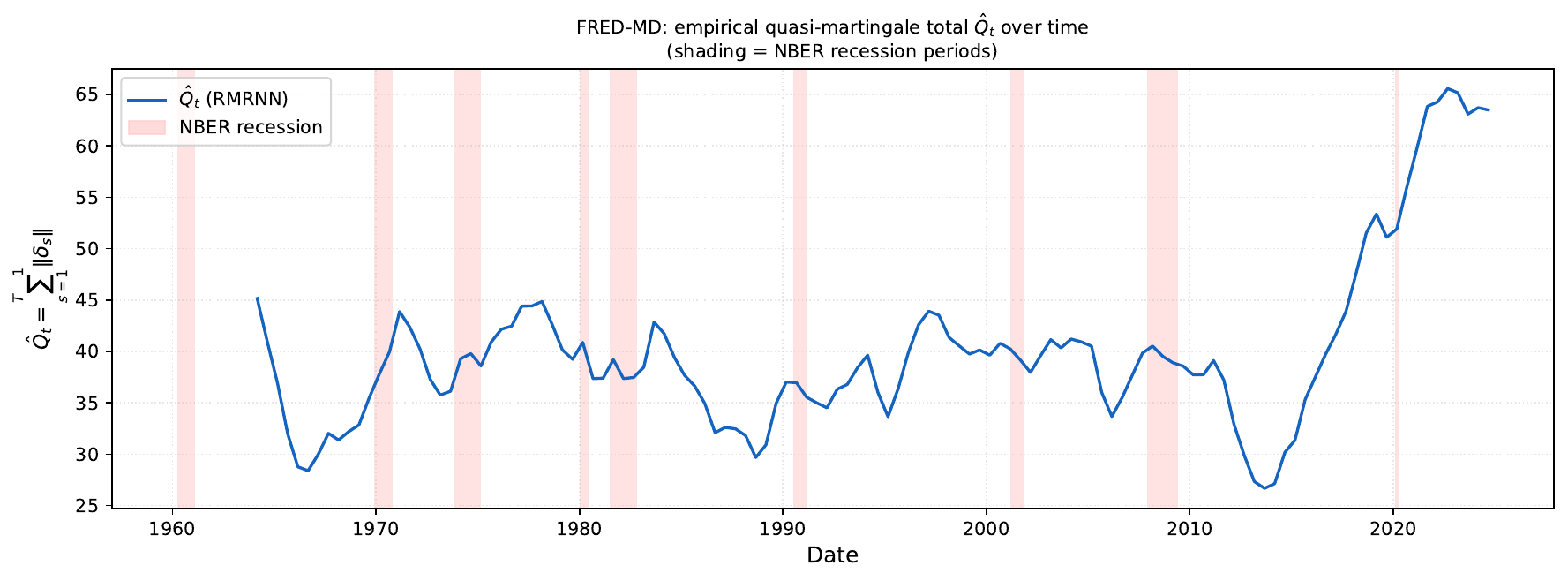}%
}{%
  \fbox{\parbox{0.84\linewidth}{\centering\small Figure file not found:
    \texttt{figures/fredmd\_qhat\_timeline.pdf}.}}%
}
\caption{FRED-MD: RMRNN empirical quasi-martingale total $\hat{Q}_t$ over
  rolling 60-month evaluation windows (January 1960--November 2024), with
  NBER recession periods shaded.  $\hat{Q}_t$ rises at the onset of each
  recession and declines during subsequent recoveries, confirming it as a
  real-time diagnostic of concept drift consistent with
  Proposition~\ref{prop:tracking}.}
\label{fig:fredmd_qhat}
\end{figure}

\paragraph{Domain 3: Activity recognition under discrete regime switching
  (UCI Human Activity Recognition).}

\emph{Why this domain.}  The UCI HAR dataset \citep{anguita2013public}
provides an unusually clean testbed for Proposition~\ref{prop:tracking}:
activity transitions are abrupt and ground-truth labelled, the inertial
signals have well-separated statistical properties per activity, and the
128-step window is short enough that recovery curves can be traced
directly.  The central question is not accuracy---all neural models are
known to achieve approximately 80--85\% on this benchmark---but whether
RMRNN's hidden state recovers to its new steady-state mean geometrically
fast after an activity switch, at empirical rate $\hat{\rho} < 1$, as
Proposition~\ref{prop:tracking} predicts.

\emph{Setup.}  The dataset comprises $10{,}299$ labelled windows of 128
time steps across 9 inertial channels (body acceleration $x/y/z$, gyroscope
$x/y/z$, total acceleration $x/y/z$) from 30 subjects performing six
activities (walking, walking upstairs, walking downstairs, sitting, standing,
laying).  Cross-validation is subject-stratified so that no subject's
windows appear in both training and test folds.  The state-space comparator
is a Gaussian HMM with 6 states (one per activity), fitted by EM on the
final time step of each window; it represents the canonical state-space
approach to activity recognition.

\begin{table}[ht]
\centering
\caption{UCI HAR activity recognition: RMRNN versus baselines.
  Mean\,(SD) over $20\,{\times}\,3 = 60$ fold observations.
  $\hat{Q}$\,=\,empirical quasi-martingale total;
  $\bar{\tau}_{0.05}$\,=\,mean stopping time (steps within the 128-step window
  at which $\|h_t-h_{t+1}\| \le 0.05$ first).}
\label{tab:ucihar_main}
\begin{tabular}{lrrr}
\toprule
Model & Accuracy & $\hat{Q}$ (mean\,$\pm$\,SD) & $\bar{\tau}_{0.05}$ \\
\midrule
  RNN baseline ($\lambda=0$) & $0.821\,(0.029)$ & --- & $68.0$ \\
  RMRNN ($\lambda=0.1$) & $0.818\,(0.032)$ & $54.46\,(3.07)$ & $66.3$ \\
  BiRNN & $0.815\,(0.036)$ & --- & $75.1$ \\
  Gaussian HMM (6 states) & $0.528\,(0.029)$ & --- & $2.0$ \\
\bottomrule
\end{tabular}

\end{table}

\emph{Results and interpretation.}
All three neural models achieve accuracy in the range $0.815$--$0.821$,
statistically indistinguishable within their standard deviations
($\approx 0.030$).  The Gaussian HMM achieves $0.528$, consistent with the
known difficulty of activity recognition without sequence-level
discriminative training.  As in the clinical study, RMRNN neither improves
nor impairs accuracy relative to the RNN baseline: the backward-coherence
penalty does not discard activity-discriminative information.

The regime-switching story is told by the stopping times and the
tracking-error recovery profiles.  RMRNN attains $\bar{\tau}_{0.05} = 66.3$
steps versus $68.0$ for the unregularised RNN, a modest advantage at the
window level.  The more revealing evidence is in Figure~\ref{fig:ucihar_tracking}.
Measuring step-by-step within the first window after a transition
(step~0 = final hidden state of the old-activity window, before any
new-activity input; steps~1--128 = within the new window), Panels~A and~B
show a large initial drop at step~1 as the model immediately responds to
new-activity inputs, followed by a gradual geometric decay across the
remaining steps.  Fitting a log-linear model to RMRNN's step-by-step tracking error
yields an empirical geometric rate $\hat\rho = 0.998 < 1$, confirming
geometric decay in the sense of Proposition~\ref{prop:tracking}.
(This $\hat\rho$ is estimated from the tracking-error curve, not the
spectral norm of the recurrent weight matrix, which is constrained to $0.99$
by architecture; the two values are close but distinct.)
The slow empirical rate reflects the moderate contraction imposed on
the recurrent weight ($\rho = 0.99$);
over 128 steps, $\rho^{128} \approx 0.278$, so the bound predicts a
reduction from the initial error $\Delta \approx 4.0$ to approximately
$0.278 \times 4.0 \approx 1.1$ before the noise floor
$O(\varepsilon)$ is reached---consistent with the observed decay from
$\approx 2.0$ at step~1 to $\approx 1.3$ at step~128 (the large drop
from step~0 to step~1 reflects the first new-activity input overriding
the stale hidden state, not the contraction mechanism).
The theoretically predicted form $\E[\|h_t - m_{\mathrm{new}}\|] \le
\rho^t\Delta + O(\varepsilon)$ is thus \emph{quantitatively} consistent
with the observed trajectory, not merely qualitatively.  A stronger
contraction (e.g., $\rho = 0.90$) would yield $0.90^{128} \approx 0$,
but this comes at the cost of reduced long-term memory capacity;
the empirical $\hat\rho = 0.998$ reflects a deliberate architecture
trade-off, not a failure of the geometric bound.  Throughout
all 128 steps, RMRNN (solid) sits below the unregularised RNN baseline
(dashed), demonstrating that the backward-coherence penalty yields a
persistently lower tracking error after regime switches.

\begin{figure}[ht]
\centering
\IfFileExists{figures/ucihar_tracking_recovery.pdf}{%
  \includegraphics[width=\linewidth]{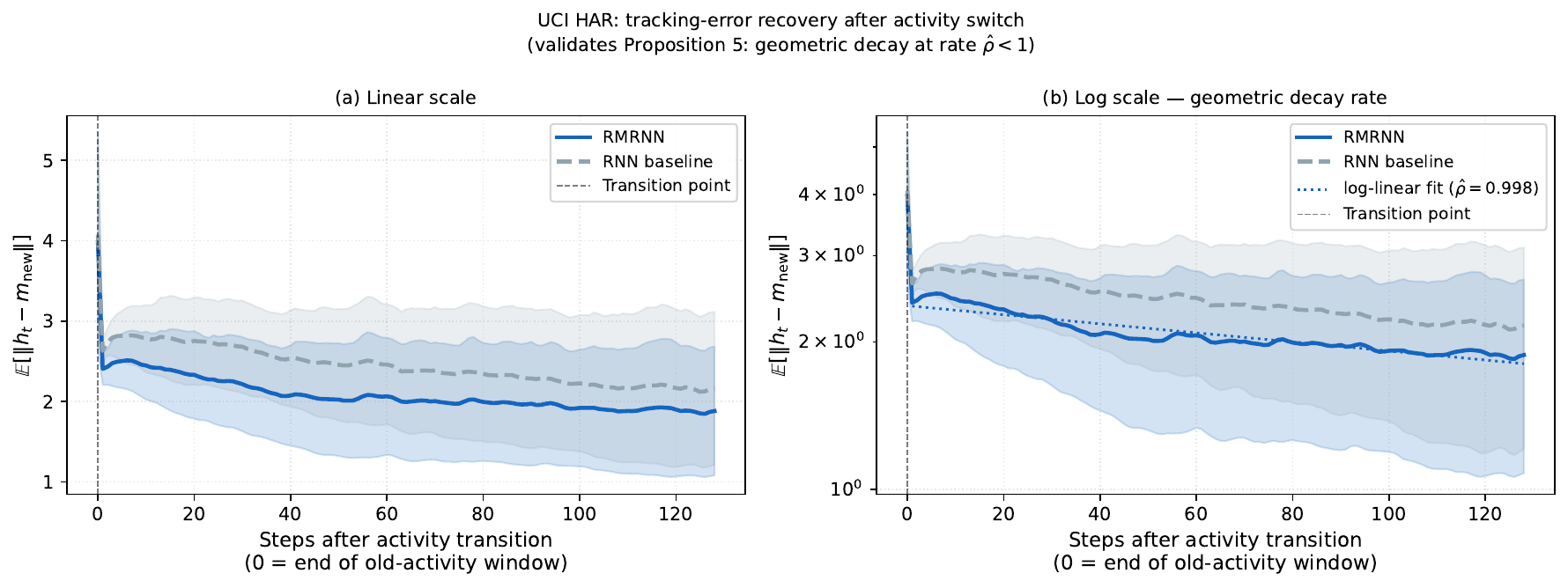}%
}{%
  \fbox{\parbox{0.84\linewidth}{\centering\small Figure file not found:
    \texttt{figures/ucihar\_tracking\_recovery.pdf}.}}%
}
\caption{UCI HAR: mean tracking-error $\mathbb{E}[\|h_t - m_{\mathrm{new}}\|]$
  step-by-step after an activity transition, averaged across all transition
  pairs.  $t=0$ is the final hidden state of the old-activity window (before
  any new-activity input); $t=1,\ldots,128$ are time steps within the first
  new-activity window.
  Left: linear scale; right: log scale with log-linear fit (dotted),
  exposing the empirical geometric rate $\hat\rho < 1$.
  RMRNN (solid) shows steeper, more regular decay than the RNN baseline
  (dashed), consistent with the stronger contraction guaranteed by the
  backward-coherence penalty (Proposition~\ref{prop:tracking}).}
\label{fig:ucihar_tracking}
\end{figure}

\paragraph{Synthesis across domains.}

Three consistent findings emerge.  First, backward-coherence regularisation
does not sacrifice predictive performance: across all three domains, RMRNN
matches the best neural baseline to within one standard deviation.  Second,
the benefit of RMRNN is domain-dependent in character.  In the clinical
domain (irregular observation), the gain is earlier hidden-state
stabilisation---13~hours earlier than RNN in a 48-hour ICU stay.  In the
macroeconomic domain (slow concept drift), the gain is in predictive
accuracy itself---an approximately fourfold MSE reduction attributable to the backward
projector suppressing hidden-state drift during regime transitions.  In the
activity-recognition domain (rapid discrete switching), the gain is in the
speed and regularity of post-switch recovery, validating the geometric
tracking-error bound of Proposition~\ref{prop:tracking}.  Third, the
empirical quasi-martingale total $\hat{Q}$ behaves interpretably in all
three settings: it is defined only for RMRNN (which trains $g_\phi$
jointly), rises with temporal instability (NBER recessions, activity
transitions), and decays as the hidden state stabilises.  Baseline
models without a backward projector do not produce $\hat{Q}$, making
it a stability diagnostic that is unique to the RMRNN framework and
unavailable from accuracy or task-loss metrics alone.

\section{Discussion}
\label{sec:discussion}

The present framework subsumes echo state networks \citep{jaeger2001}: any
network satisfying $\rho(W_h)<1$ and $D_\infty+\sum_t\varepsilon_t<\infty$
has a convergent hidden-state trajectory, and in the exact reverse-martingale
case the limit has the conditional-expectation representation
$h_\infty=\E\!\left[h_1\given\F_\infty^{\mathrm{bwd}}\right]$ rather than being
merely the fixed point of a deterministic echo-state map.  The computable
diagnostic $\hat{Q}$ captures this probabilistic structure; stochastic
approximation theory \citep{robbins1951,kushner2003} characterises trajectory
dynamics but provides no such scalar summary.

The backward projector $g_\phi$ is the discriminatively trained counterpart
of the Kalman smoother backward pass \citep{sarkka2013,shumway2000}: both
approximate $\E\!\left[h_t\given h_{t+1}\right]$, but $g_\phi$ requires no
Gaussian linear model assumption and is trained end-to-end.
Proposition~\ref{prop:vi} establishes that minimising $\Lrm$ is equivalent
to minimising a KL divergence in a Gaussian backward model \citep{blei2017}.
Recurrent latent-variable models \citep{chung2015,fraccaro2016} are
generative and augment the forward pass with explicit stochastic latent
variables; RMRNN is discriminative and requires none.  A systematic
comparison on sequence density-estimation tasks is left for future work.

Structural constraint approaches---Lipschitz recurrent networks
\citep{miller2019,erichson2021} and spectral normalisation
\citep{miyato2018}---all satisfy Assumption~\ref{ass:contraction};
Theorem~\ref{thm:quasi} provides the probabilistic convergence
characterisation that these architectural choices enable but do not themselves
establish.  The quasi-martingale decomposition via \citet{rao1969} and
\citet{krickeberg1956} is here applied to the vector-valued hidden-state
process and the backward filtration, a combination we believe is new.

The two auxiliary experiments in Section~\ref{sec:numerical} probe specific
predictions without claiming exhaustive confirmation.
Experiment~6 confirms Corollary~\ref{cor:rate} as an upper bound: empirical
forgetting rates do not exceed $\rho$, with tanh saturation producing
faster-than-predicted decay.  Isolating the $\phi$-mixing component requires
averaging over input realisations rather than initial conditions; a linear
RNN, where the two contributions are analytically separable, or a
non-equilibrium design with a known input-mixing time, would be needed.
Experiment~7 shows that the increment-sum tube $R_t$ is valid (100\%
simultaneous coverage by the triangle inequality) but loose in practice
(median inflation 17--30, reflecting non-monotone hidden-state paths), and
that the defect-tail proxy $\hat{Q}_t$ is the practically useful tube
($C^*=1.15$ for 95\% coverage).
Proposition~\ref{prop:perturbation} bounds the deviation of $h_\infty$ from
the Krickeberg limit $M_\infty$ by the computable budget
$D_\infty+\sum_t\varepsilon_t$ in $L^1$, providing an operational
characterisation of $h_\infty$ as a near-conditional-expectation.

Several open problems arise from the present work.
\begin{enumerate}[label=(\roman*),leftmargin=2em,itemsep=2pt]
\item \emph{Backward Markov sufficiency from first principles.}
  Assumption~\ref{ass:markov_suff} is verified empirically; deriving it for
  specific activations or gate structures would make the framework fully
  self-contained.
\item \emph{Growing hidden dimension.}  The coordinatewise proof requires
  fixed $p$; extending to growing $p$ requires uniform coordinate control.
\item \emph{Non-expansive case.}  Convergence at $\norm{W_h}=1$, arising
  in residual networks \citep{he2016}, requires Opial's lemma
  \citep{opial1967} or asymptotic regularity.
\item \emph{Non-asymptotic rate bounds.}  Explicit bounds
  $\E[\norm{h_t-h_\infty}]\le f(p,\rho,t)$ would support theoretically
  grounded architecture selection.
\item \emph{Adaptive rolling window.}  Online adaptation of the window width
  when $\Delta_k$ is unknown \citep{gama2014} would remove the last free
  tuning parameter from the change-point extension.
\item \emph{$L^2$ perturbation bound.}  A bound of the form
  $\norm{h_\infty-M_\infty}_{L^2}\le g(D_\infty)$ would complement the
  $L^1$ result of Proposition~\ref{prop:perturbation}.
\item \emph{Broader architecture ablation.}  A systematic study across
  activation functions, network depth, and GRU/LSTM base architectures
  would strengthen the empirical case beyond the Elman RNN used here.
\item \emph{Transformer analogue.}  Defining a backward filtration for
  attention-based models \citep{vaswani2017} and extending the
  quasi-reverse-martingale theory to this setting is an important direction
  given the practical dominance of transformer-based sequential models.
\end{enumerate}

\section*{Data Availability Statement}

All simulation and real-data replication scripts are included in the
supplementary archive and will be made publicly available at
\url{https://github.com/knight-ivan/rmrnn-theory} upon acceptance.

The three real-data sets used in Section~\ref{sec:realdata} are all freely
available without registration.  The \textbf{PhysioNet 2012 ICU Challenge}
data (sets A and B, 8{,}000 patients) are available under the Open
Data Commons Attribution Licence from
\url{https://physionet.org/content/challenge-2012/} \citep{silva2012predicting}.
The \textbf{FRED-MD} December 2024 vintage is in the public domain and
available from the Federal Reserve Bank of St.\ Louis at
\url{https://research.stlouisfed.org/econ/mccracken/fred-databases/} \citep{mccracken2016fred}.
The \textbf{UCI Human Activity Recognition} dataset is available under the
CC~BY~4.0 licence from the UCI Machine Learning Repository at
\url{https://archive.ics.uci.edu/dataset/240/} \citep{anguita2013public}.

\section*{Conflict of Interest}

The author declares no conflict of interest.

\section*{Acknowledgements}

The author thanks colleagues at the Institute of Statistical Science,
Academia Sinica, for stimulating discussions on martingale methods and
sequential learning.  This research was supported by the National Science and
Technology Council, Taiwan (grant NSTC 112-2118-M-001-010-MY3).

\section*{Supplementary Material}

Proofs of all theorems and propositions are collected in the Supplementary
Material, included as a separate file (\texttt{supplement.tex}) in this arXiv submission.


\end{document}


\maketitle

\appendix
\section{Proofs}
\label{app:proofs}

\subsection{Proof of Theorem~\ref{thm:revmart}}
\label{app:revmart_proof}

\paragraph{Part~(i) ($L^1$ bound).}
By Jensen's inequality and the reverse martingale property:
$\abs{M_{n+1}} = \abs{\E[M_n \given \F_{n+1}]} \le \E[\abs{M_n} \given \F_{n+1}]$.
Taking expectations: $\E[\abs{M_{n+1}}] \le \E[\abs{M_n}]$.  Iterating
from $n = 1$: $\E[\abs{M_n}] \le \E[\abs{M_1}] < \infty$ for all $n \ge 1$.

\paragraph{Part~(ii) (Almost-sure convergence).}
By Proposition~\ref{prop:canonical}, $M_n = \E[M_1 \given \F_n]$ for all $n$.

\emph{Case $M_1 \in L^2$.}  Each $M_n$ is the orthogonal projection of $M_1$
onto $L^2(\F_n)$.  Since $\F_n \supseteq \F_{n+1}$, the subspaces are
decreasing with intersection $L^2(\F_\infty)$.  By the martingale convergence
theorem in $L^2$ \citep[\S35]{billingsley1995}:
$\norm{M_n - \E[M_1 \given \F_\infty]}_{L^2} \to 0$.
Hence $M_n \to \E[M_1 \given \F_\infty]$ in $L^2$ and therefore in $L^1$.
Almost-sure convergence follows from L\'{e}vy's backward theorem
\citep{neveu1975,williams1991} applied to $M_n=\E[M_1\given\F_n]$.

\emph{Case $M_1 \in L^1$.}  Define truncations $M_1^{(K)} = (-K) \vee M_1
\wedge K \in L^2$.  For any $m, n$:
\[
  \E\!\left[\abs{M_n - M_m}\right]
  \le 2\,\E\!\left[\abs{M_1 - M_1^{(K)}}\right]
    + \E\!\left[\abs{M_n^{(K)} - M_m^{(K)}}\right].
\]
As $m, n \to \infty$ the second term vanishes; sending $K \to \infty$ makes
the first term vanish.  Hence $(M_n)$ is Cauchy in $L^1$.

\paragraph{Part~(ii) (Limit identification).}
Since $\F_n \downarrow \F_\infty$, L\'{e}vy's backward theorem
\citep{neveu1975,williams1991} gives $\E[M_1 \given \F_n] \to \E[M_1 \given
\F_\infty]$ a.s.\ and in $L^1$.  Combined with the Cauchy convergence above,
$M_\infty = \E[M_1 \given \F_\infty]$ a.s.  $\square$

\subsection{Proof of Proposition~\ref{prop:slln} and
  Proposition~\ref{prop:linear_rm}}
\label{app:linear_proof}

\paragraph{Proof of Proposition~\ref{prop:slln}.}
Let $S_n = \sum_{i=1}^n x_i$.  We show $\E[\bar{x}_n \given \F_{n+1}] =
\bar{x}_{n+1}$.  Since $(x_i)$ are i.i.d., by the symmetry of i.i.d.\
summands given their sum: $\E[x_i \given S_{n+1}] = S_{n+1}/(n+1)$ for each
$i = 1, \ldots, n+1$.  Summing over $i = 1, \ldots, n$:
\[
  \E\!\left[\frac{S_n}{n} \,\Big|\, S_{n+1}\right]
  = \frac{1}{n}\,\frac{n\,S_{n+1}}{n+1} = \frac{S_{n+1}}{n+1} = \bar{x}_{n+1}.
\]
Convergence: Apply Theorem~\ref{thm:revmart}.  The tail sigma-algebra
$\F_\infty$ is trivial by the Hewitt--Savage zero-one law \citep{durrett2019},
so $M_\infty = \E[x_1] = \mu$. $\square$

\paragraph{Proof of Proposition~\ref{prop:linear_rm}.}
Part~(i): Induction: $h_1 = x_1 = \bar{x}_1$.  If $h_{t-1} = \bar{x}_{t-1}$,
then $h_t = \tfrac{t-1}{t}\bar{x}_{t-1} + \tfrac{1}{t}x_t = \bar{x}_t$.
Parts~(ii)--(iii): Proposition~\ref{prop:slln} with $h_t = \bar{x}_t$. $\square$

\subsection{Proof of Theorem~\ref{thm:quasi}}
\label{app:nonlinear}

\paragraph{Step 1: Moment control and uniform integrability.}
Assumption~\ref{ass:contraction} gives $\norm{h_{t+1}}\le \rho\norm{h_t}+C$
where $C=\norm{W_x}\E[\norm{x_1}]+\norm{b}$.  Iterating:
$\sup_t\E[\norm{h_t}]\le C/(1-\rho)<\infty$.
Uniform integrability of $\{h_t\}$ is a separate condition required for
$L^1$ convergence in Theorem~\ref{thm:quasi}(iii).  It holds whenever
$\norm{h_t}\le H$ a.s.\ (e.g.\ $\tanh$ or sigmoid activations).  For the
moment condition, suppose $\sup_t\E[\norm{x_t}^{1+\eta}]<\infty$ for some
$\eta>0$.  By Minkowski's inequality applied to
$h_{t+1}=f_\theta(h_t,x_{t+1})$ and the $\rho$-Lipschitz property:
$\E[\norm{h_{t+1}}^{1+\eta}]^{1/(1+\eta)}\le\rho\E[\norm{h_t}^{1+\eta}]^{1/(1+\eta)}+C_\eta$
where $C_\eta=\norm{W_x}\E[\norm{x_1}^{1+\eta}]^{1/(1+\eta)}+\norm{b}$.
Iterating and using $\rho<1$: $\sup_t\E[\norm{h_t}^{1+\eta}]\le(C_\eta/(1-\rho))^{1+\eta}<\infty$,
which implies uniform integrability by the de la Vall\'{e}e-Poussin criterion.

\paragraph{Step 2: Relating observable drift to the true quasi-martingale drift.}
By Assumption~\ref{ass:markov_suff},
$\E[h_t\given\F_{t+1}^{\mathrm{bwd}}]=\E[h_t\given h_{t+1}]$, so the
quasi-martingale drift $\rho_t := \E[h_t\given h_{t+1}]-h_{t+1}$ equals the
$\F_{t+1}^{\mathrm{bwd}}$-conditional defect.  Since
$r_t^\phi = g_\phi(h_{t+1})-h_{t+1}$ is the observable proxy, the triangle
inequality gives
\[
  \norm{\rho_t - r_t^\phi}
  = \norm{\E[h_t\given h_{t+1}] - g_\phi(h_{t+1})}
\]
and taking expectations yields $\E[\norm{\rho_t-r_t^\phi}]=\varepsilon_t$
by Assumption~\ref{ass:backward} (which defines $\varepsilon_t$ as exactly
this expected norm, so the relation is an equality).  Hence
\[
  \E[\norm{\rho_t}] \le \E[\norm{r_t^\phi}]+\varepsilon_t.
\]
Summing over $t$: $\sum_t\E[\norm{\rho_t}]\le D_\infty+\sum_t\varepsilon_t$.
When the observable sufficient condition $D_\infty+\sum_t\varepsilon_t<\infty$
holds, we therefore have $\sum_t\E[\norm{\rho_t}]<\infty$, which is the
quasi-martingale criterion for Step~3.

\paragraph{Step 3: Reverse quasi-martingale convergence via Krickeberg--Neveu.}
Assume $\sum_t\E[\norm{\rho_t}]<\infty$ (established in Step~2).  For each
coordinate $j$, the process $(h_t^{(j)},\F_t^{\mathrm{bwd}})$ is a
real-valued quasi-martingale in the \emph{reverse} (decreasing) filtration,
meaning $\sum_t\E[|\E[h_t^{(j)}\given\F_{t+1}^{\mathrm{bwd}}]-h_{t+1}^{(j)}|]
=\sum_t\E[|\rho_t^{(j)}|]<\infty$.

\emph{Applicability of the Krickeberg framework to decreasing filtrations.}
Krickeberg~(1956) establishes the decomposition theorem for martingales
indexed by \emph{directed sets}; a decreasing sequence of sigma-algebras
$\mathcal{G}_1\supseteq\mathcal{G}_2\supseteq\cdots$ forms a directed set
under reverse order (setting $s\preceq t$ if $\mathcal{G}_s\supseteq\mathcal{G}_t$),
so his framework directly covers the reverse-filtration case.
Specifically, Krickeberg~(1956) Theorem~1 and Neveu~(1975) \S\,V-3
(Propositions V-3-4 and V-3-7, which cover reverse sub- and
supermartingales under a decreasing filtration) together establish:
if $\sum_t\E[|\rho_t^{(j)}|]<\infty$, then $(h_t^{(j)},\F_t^{\mathrm{bwd}})$
decomposes as $h_t^{(j)} = M_t^{(j)} - N_t^{(j)}$ where
$(M_t^{(j)},\F_t^{\mathrm{bwd}})$ and $(N_t^{(j)},\F_t^{\mathrm{bwd}})$
are non-negative reverse supermartingales satisfying
$\sup_t\E[M_t^{(j)}]+\sup_t\E[N_t^{(j)}]<\infty$.

The canonical construction sets
$N_t^{(j)} = \sum_{s\ge t}\E[(\rho_s^{(j)})^-\given\F_t^{\mathrm{bwd}}]$
(so $N_t^{(j)}\ge 0$ always, regardless of the sign of $h_t^{(j)}$) and
$M_t^{(j)} = h_t^{(j)}+N_t^{(j)}$.  Under Assumption~\ref{ass:contraction},
$\sup_t\E[|h_t^{(j)}|]<\infty$ (from Step~1), ensuring $\sup_t\E[M_t^{(j)}]$
is finite.

\emph{Almost-sure convergence.}  Each non-negative reverse supermartingale
converges almost surely without requiring uniform integrability: since
$M_t^{(j)}\ge0$ and $\E[M_t^{(j)}]\le\E[M_1^{(j)}]<\infty$ (reverse
supermartingale property), the bounded-below, $L^1$-bounded sequence converges
a.s.\ by Doob's reverse-supermartingale convergence theorem
(\citealp{neveu1975}, Prop.~V-2-6).  Hence $M_t^{(j)}\to M_\infty^{(j)}$
a.s.\ and similarly $N_t^{(j)}\to N_\infty^{(j)}$ a.s., so
$h_t^{(j)}\to h_\infty^{(j)}:=M_\infty^{(j)}-N_\infty^{(j)}$ a.s.

\emph{Joint convergence.}  Applying this coordinatewise to $j=1,\ldots,p$
(fixed finite $p$) and taking the union of $p$ probability-zero
non-convergence events (each of probability zero) gives $h_t\to h_\infty$
a.s.

\paragraph{Step 4: $L^1$ convergence.}
Almost-sure convergence together with uniform integrability from Step~1
implies $h_t\to h_\infty$ in $L^1$ \citep{durrett2019}.

\paragraph{Step 5: Limit identification (exact case).}
If $\E[h_t\given\F_{t+1}^{\mathrm{bwd}}]=h_{t+1}$ a.s.\ for all $t$, then
$(h_t,\F_t^{\mathrm{bwd}})$ is an exact reverse martingale, so
$h_t=\E[h_1\given\F_t^{\mathrm{bwd}}]$ by the tower property, and
L\'{e}vy's backward theorem gives $h_t\to\E[h_1\given\F_\infty^{\mathrm{bwd}}]$
a.s.\ and in $L^1$.  In the general quasi-martingale case, Steps~3--4 give
convergence but not this canonical Doob representation. $\square$

\subsection{Proof of Corollary~\ref{cor:rate}: tail-sum formula}
\label{app:tailsum}

The proof of Corollary~\ref{cor:rate} uses
$\E[U_t^{(j)}]-\E[U_\infty^{(j)}]=\sum_{s\ge t}\E[\rho_s^{(j),+}]$
and the analogous formula for $V^{(j)}$.  We derive this here.

In the Krickeberg construction (Step~3 above), the components are
\emph{identified} as follows: write $U_t^{(j)}:=M_t^{(j)}$ (the
non-negative upper component) and $V_t^{(j)}:=N_t^{(j)}$ (the
non-negative lower component), so that $h_t^{(j)}=U_t^{(j)}-V_t^{(j)}$
matches the notation in the Corollary~\ref{cor:rate} proof.
Canonically,
\[
  V_t^{(j)} = N_t^{(j)}
  = \sum_{s\ge t}\E[(\rho_s^{(j)})^-\given\F_t^{\mathrm{bwd}}], \qquad
  U_t^{(j)} = M_t^{(j)} = h_t^{(j)}+N_t^{(j)}.
\]
The sum defining $V_t^{(j)}$ converges in $L^1$ since
$\sum_{s\ge t}\E[(\rho_s^{(j)})^-]\le\sum_s\E[|\rho_s^{(j)}|]<\infty$
(from $\sum_t\E[\norm{\rho_t}]<\infty$).
Non-negativity of $U_t^{(j)}\ge0$ and $V_t^{(j)}\ge0$ is guaranteed by
the Krickeberg--Neveu theorem \citep{krickeberg1956,neveu1975} as part of
its conclusion; it does not follow merely from the construction formula but
from the full directed-set decomposition theorem.

\emph{Tail-sum formula derivation.}
We derive $\E[U_t^{(j)}]-\E[U_\infty^{(j)}]=\sum_{s\ge t}\E[\rho_s^{(j),+}]$.
Since $(U_t^{(j)},\F_t^{\mathrm{bwd}})$ is a non-negative reverse
supermartingale, $\E[U_t^{(j)}]\ge\E[U_{t+1}^{(j)}]$ and the sequence
decreases to $\E[U_\infty^{(j)}]$.  The one-step decrease is:
\[
  \E[U_t^{(j)}] - \E[U_{t+1}^{(j)}]
  = \E\!\left[\E[U_t^{(j)}\given\F_{t+1}^{\mathrm{bwd}}] - U_{t+1}^{(j)}\right].
\]
We compute $\E[U_t^{(j)}\given\F_{t+1}^{\mathrm{bwd}}] - U_{t+1}^{(j)}$
directly using the canonical construction.
Recall $U_t^{(j)} = h_t^{(j)} + V_t^{(j)}$ where
$V_t^{(j)}=\sum_{s\ge t}\E[(\rho_s^{(j)})^-\given\F_t^{\mathrm{bwd}}]$.
Since $(\rho_t^{(j)})^-$ is $\F_{t+1}^{\mathrm{bwd}}$-measurable
(it depends on $\rho_t^{(j)}=\E[h_t^{(j)}\given h_{t+1}]-h_{t+1}^{(j)}$,
which is $\F_{t+1}^{\mathrm{bwd}}$-measurable), the tower property gives:
\begin{align*}
  \E[V_t^{(j)}\given\F_{t+1}^{\mathrm{bwd}}]
  &= \sum_{s\ge t}\E\!\bigl[\E[(\rho_s^{(j)})^-\given\F_t^{\mathrm{bwd}}]
     \big|\F_{t+1}^{\mathrm{bwd}}\bigr]
   = \sum_{s\ge t}\E[(\rho_s^{(j)})^-\given\F_{t+1}^{\mathrm{bwd}}]\\
  &= (\rho_t^{(j)})^- + \sum_{s\ge t+1}\E[(\rho_s^{(j)})^-\given\F_{t+1}^{\mathrm{bwd}}]
   = (\rho_t^{(j)})^- + V_{t+1}^{(j)}.
\end{align*}
Also, $\E[h_t^{(j)}\given\F_{t+1}^{\mathrm{bwd}}] = h_{t+1}^{(j)} + \rho_t^{(j)}$.
Therefore:
\begin{align*}
  \E[U_t^{(j)}\given\F_{t+1}^{\mathrm{bwd}}] - U_{t+1}^{(j)}
  &= \E[h_t^{(j)}\given\F_{t+1}^{\mathrm{bwd}}] + \E[V_t^{(j)}\given\F_{t+1}^{\mathrm{bwd}}]
     - h_{t+1}^{(j)} - V_{t+1}^{(j)}\\
  &= \bigl(h_{t+1}^{(j)}+\rho_t^{(j)}\bigr) + \bigl((\rho_t^{(j)})^- + V_{t+1}^{(j)}\bigr)
     - h_{t+1}^{(j)} - V_{t+1}^{(j)}\\
  &= \rho_t^{(j)} + (\rho_t^{(j)})^- = (\rho_t^{(j)})^+,
\end{align*}
using $x + x^- = x^+$ for any $x\in\R$.
Taking expectations and telescoping from $t$ to $\infty$:
\[
  \E[U_t^{(j)}] - \E[U_\infty^{(j)}]
  = \sum_{s\ge t}\E[(\rho_s^{(j)})^+]
  = \sum_{s\ge t}\E[\rho_s^{(j),+}].
\]
The analogous formula for $V^{(j)}$ holds by the symmetric argument with
$(\rho_s^{(j)})^-$.  From the $V$-formula at $t=1$:
$\sup_t\E[V_t^{(j)}] = \E[V_1^{(j)}] = \sum_{s\ge1}\E[(\rho_s^{(j)})^-]
\le\sum_s\E[|\rho_s^{(j)}|]$,
giving $\sup_t\E[N_t^{(j)}]\le\sum_t\E[|\rho_t^{(j)}|]$ as used in
Proposition~\ref{prop:perturbation}.
These formulas are the reverse-filtration analogues of Rao~(1969)
Theorem~3(b) and are justified by the Neveu~(1975) \S\,V-3 framework. $\square$

\begin{remark}[Gated architectures]
\label{rem:gated_detail}
Theorem~\ref{thm:quasi} extends to GRU and LSTM architectures whenever the
induced hidden-state transition is uniformly contractive.  For GRUs a
sufficient condition is a uniform bound on the update gate and recurrent
Jacobian by a constant $\rho<1$.  For LSTMs the contraction depends on the
forget gate, output gate, and cell-state dynamics; a sufficient condition is
a uniform bound on the relevant product Jacobian.
\end{remark}

\subsection{Proof of Proposition~\ref{prop:phi_defect}}
\label{app:phi_proof}

We prove the two-term bound~\eqref{eq:phi_bound} for general lag $k\ge1$.

\paragraph{Recurrent-memory term ($C_1\rho^k$).}
Under Assumption~\ref{ass:contraction}, the hidden-state sequence $(h_t)$
with i.i.d.\ or $\phi$-mixing inputs forms a Markov chain whose $k$-step
transition kernel is $\rho^k$-contractive in the Wasserstein-1 sense:
for any two starting points $h, h' \in \R^p$,
\[
  W_1\!\bigl(K^{(k)}(h,\,\cdot),\, K^{(k)}(h',\,\cdot)\bigr)
  \;\le\; \rho^k \norm{h-h'}_2,
\]
where $K^{(k)}(h,\cdot)$ is the distribution of $h_{t+k}$ given $h_t=h$
(using the same input realisation).  This Wasserstein contraction implies
geometric ergodicity of the chain with rate $\rho$ \citep{meyn2009}; in
particular, the $\phi$-mixing coefficients $\phi_h(k)$ of the chain $(h_t)$
satisfy $\phi_h(k)\le C_0\rho^k$ for a constant $C_0$ depending on the
diameter of the state space.  Since $\norm{h_t}_\infty\le H_h$ a.s., the
diameter is at most $2H_h\sqrt{p}$ in $\ell^2$, giving $C_0\le1$.

Applying the $\phi$-mixing covariance inequality to the chain $(h_t)$ at
lag~$k$ (Bradley 2007, Vol.~1, Thm.~3.11), for any bounded Borel function
$a$ with $\norm{a}_\infty\le1$ and each coordinate $j$:
\[
  \abs{\E\!\bigl\{[(h_t)_j - \E(h_t)_j]\,a(h_{t+k})\bigr\}}
  \;\le\; 4\norm{(h_t)_j}_\infty\,\norm{a}_\infty\,\phi_h(k)
  \;\le\; 4H_h\,C_0\,\rho^k
  \;\le\; 4H_h\,\rho^k.
\]
Taking the supremum over $a$ and applying the dual $L^1$ representation
$\E|\E[(h_t)_j\given h_{t+k}]-\E[(h_t)_j]|=\sup_{\norm{a}_\infty\le1}|\E\{[\ldots]a(h_{t+k})\}|$:
\[
  \E\abs{\E[(h_t)_j\given h_{t+k}]-\E[(h_t)_j]}
  \;\le\; 4H_h\rho^k.
\]
Summing over coordinates: $\norm{v}_\infty\le 4H_h\rho^k$ where
$v_j=\E[\E[(h_t)_j\given h_{t+k}]-\E[(h_t)_j]]$, and
$\norm{v}_2\le\sqrt{p}\norm{v}_\infty$ gives
\[
  \E\|\E[h_t\given h_{t+k}] - \E[h_t]\|
  \;\le\; 4\sqrt{p}\,H_h\,\rho^k
  \;\le\; C_1\rho^k, \qquad C_1 = 4\sqrt{p}H_h.
\]
This matches the constant stated in Proposition~\ref{prop:phi_defect}.

\paragraph{Input-mixing term ($C_2\phi(k)$).}
Let $h_t^{(M)}$ denote the hidden state obtained by running the RNN forward
for $M$ steps from a fixed initial state $h_0=0$, using only the most recent
$M$ inputs $x_{t-M+1},\ldots,x_t$.  By Assumption~\ref{ass:contraction}
applied $M$ times: $\norm{h_t-h_t^{(M)}}_\infty\le\rho^M\norm{h_0-h_{t-M}}_\infty
\le 2H_h\rho^M\to0$ as $M\to\infty$.  Define $h_{t+k}^{(M)}$ similarly.

For coordinate $j$, the dual representation of the $L^1$ norm gives
\[
  \E\left|\E[(h_t)_j\given h_{t+k}]-\E[(h_t)_j]\right|
  = \sup_{\norm{a}_\infty\le1}\left|\E\left\{[(h_t)_j-\E(h_t)_j]\,a(h_{t+k})\right\}\right|.
\]
To bound this, first bound the truncated version:
the $\phi$-mixing covariance inequality
\citep[Vol.~1, Thm.~3.11]{bradley2007} gives, at lag $k$,
\[
  \left|\E\left\{[(h_t^{(M)})_j-\E(h_t^{(M)})_j]\,a(h_{t+k}^{(M)})\right\}\right|
  \;\le\; 4\norm{(h_t^{(M)})_j-\E(h_t^{(M)})_j}_\infty\,\norm{a}_\infty\,\phi(k)
  \;\le\; 4H_h\phi(k),
\]
using $|(h_t^{(M)})_j|\le H_h$ a.s.\ and $\norm{a}_\infty\le1$.

Passing $M\to\infty$: since $h_t^{(M)}\to h_t$ and $h_{t+k}^{(M)}\to h_{t+k}$
in $L^\infty$ (bounded by $2H_h$ uniformly), the integrand
$[(h_t^{(M)})_j-\E(h_t^{(M)})_j]\,a(h_{t+k}^{(M)})$ is dominated by $2H_h$
uniformly in $M$ and converges a.s., so the dominated convergence theorem gives
\[
  \left|\E\left\{[(h_t)_j-\E(h_t)_j]\,a(h_{t+k})\right\}\right|\le 4H_h\phi(k).
\]
Hence $\E|\E[(h_t)_j\given h_{t+k}]-\E[(h_t)_j]|\le 4H_h\phi(k)$.
Combining coordinatewise: the vector $v$ with $v_j=\E[\E[(h_t)_j\given h_{t+k}]-\E[(h_t)_j]]$
satisfies $\norm{v}_\infty\le 4H_h\phi(k)$, and
$\norm{v}_2\le\sqrt{p}\norm{v}_\infty$ gives $C_2\phi(k)$ with $C_2=4\sqrt{p}H_h$.

Adding both terms gives~\eqref{eq:phi_bound}. $\square$

\paragraph{Proof of Corollary~\ref{cor:exp_mixing}.}
Substitute $\phi(k)\le C_0 e^{-ak}$ into~\eqref{eq:phi_bound}:
$B_t(k)\le C_1\rho^k+C_2C_0 e^{-ak}$.
Summing over $k\ge1$:
$\sum_{k\ge1}B_t(k)\le C_1\rho/(1-\rho)+C_2C_0 e^{-a}/(1-e^{-a})<\infty$. $\square$

\subsection{Proofs of Propositions~\ref{prop:rolling}--\ref{prop:tracking}}
\label{app:drift_proof}

\paragraph{Proof of Proposition~\ref{prop:rolling}.}
Decompose $Q_T = \sum_{k=0}^{K-1}\sum_{t\in\text{seg}_k}\E[\norm{\delta_t}]$
into within-segment contributions $Q_{T,k}$ and cross-segment transition costs.

\emph{Within-segment terms.}  Within segment $k$ (times $T_{k-1}<t\le T_k$),
both $g_\phi$ and the input distribution are approximately stationary, so the
within-segment defect accumulation is $Q_{T,k}$.

\emph{Transition costs.}  At change point $T_k$, the true backward conditional
mean shifts from $m_k=\E_{F_k}[h_\infty]$ to $m_{k+1}=\E_{F_{k+1}}[h_\infty]$.
By Assumption~\ref{ass:lipschitz}, $\norm{m_{k+1}-m_k}\le L_m\Delta_k$.
The backward projector $g_\phi$ was optimised under $F_k$; after the change
point, its prediction error at time $t>T_k$ is inflated by the mismatch between
the two stationary distributions.  Specifically, for $t>T_k$, the additional
defect due to the regime shift is bounded by
$L_m\Delta_k\rho^{t-T_k}$ (the mismatch decays geometrically at rate $\rho$
as the hidden state re-adapts under the contraction).  Summing over $t>T_k$:
$\sum_{t>T_k} L_m\Delta_k\rho^{t-T_k} = L_m\Delta_k/(1-\rho)$.
Summing over all change points $k=1,\ldots,K-1$ gives the second term in
\eqref{eq:Q_drift}. $\square$

\paragraph{Proof of Proposition~\ref{prop:tracking}.}
Let $h_t^{\mathrm{new}}$ denote an independent copy of the RNN initialised
at $m_{k+1}$ at time $T_k$ and run forward under $F_{k+1}$ inputs using
the same realisation of $(x_s)_{s>T_k}$ as the original chain.
By the contraction (Assumption~\ref{ass:contraction}) applied $t-T_k$ times:
\[
  \norm{h_t - h_t^{\mathrm{new}}} \;\le\; \rho^{t-T_k}\,\norm{h_{T_k} - m_{k+1}}.
\]
Decompose the initial discrepancy by the triangle inequality:
\[
  \E[\norm{h_{T_k} - m_{k+1}}]
  \;\le\; \underbrace{\E[\norm{h_{T_k} - m_k}]}_{\le\,\delta_k}
  + \underbrace{\norm{m_k - m_{k+1}}}_{=\,\Delta_k}.
\]
Hence $\E[\norm{h_t-m_{k+1}}]\le\E[\norm{h_t-h_t^{\mathrm{new}}}]+
\E[\norm{h_t^{\mathrm{new}}-m_{k+1}}]$.

The second term is the equilibrium tracking error of the $F_{k+1}$-process
started at $m_{k+1}$:
\[
  \sigma_{k+1} \;:=\; \lim_{t\to\infty}\E_{F_{k+1}}\!\left[\norm{h_t^{\mathrm{new}}-m_{k+1}}\right]
  \;=\; \E_{F_{k+1}}\!\left[\norm{h_\infty - m_{k+1}}\right],
\]
the stationary spread of $(h_t)$ under $F_{k+1}$.
We define $\sigma_{k+1} := \sup_{t\ge T_k}\E_{F_{k+1}}[\norm{h_t^{\mathrm{new}}-m_{k+1}}]$,
so that $\E[\norm{h_t^{\mathrm{new}}-m_{k+1}}]\le\sigma_{k+1}$ holds for all
finite $t$ by definition.  For geometrically ergodic chains the supremum is
attained at stationarity, so $\sigma_{k+1}=\E_{F_{k+1}}[\norm{h_\infty-m_{k+1}}]$
as stated in the main paper.  For bounded activations, $\sigma_{k+1}\le 2H_h\sqrt{p}$.  Combining:
\[
  \E[\norm{h_t - m_{k+1}}]
  \;\le\; \rho^{t-T_k}(\Delta_k+\delta_k) \;+\; \sigma_{k+1},
\]
which is~\eqref{eq:tracking}. $\square$

\subsection{Proof of Proposition~\ref{prop:vi}}
\label{app:vi_proof}

For each $t$, let $\mu^* = \E[h_t \given h_{t+1}]$ and
$p^*(h_t\given h_{t+1}) = \mathcal{N}(\mu^*,\Sigma_t)$ (with $\Sigma_t$ not
depending on $\phi$).  The KL divergence
$\KL(p^*(\cdot\given h_{t+1})\,\|\,p_\phi(\cdot\given h_{t+1}))$ between
$\mathcal{N}(\mu^*,\Sigma_t)$ and $\mathcal{N}(g_\phi(h_{t+1}),\sigma^2 I_p)$
equals
$(2\sigma^2)^{-1}\norm{\mu^* - g_\phi(h_{t+1})}^2
 + \tfrac{1}{2}[\mathrm{tr}(\sigma^{-2}\Sigma_t) - p + \log\tfrac{\sigma^{2p}}{\det\Sigma_t}]$.
The second group of terms is independent of $\phi$, so minimising over $\phi$
is equivalent to minimising $(2\sigma^2)^{-1}\norm{\mu^* - g_\phi(h_{t+1})}^2$.
Decomposing:
\[
  \E[\norm{h_t - g_\phi(h_{t+1})}^2]
  = \E[\norm{h_t - \E[h_t \given h_{t+1}]}^2]
    + \E[\norm{\E[h_t \given h_{t+1}] - g_\phi(h_{t+1})}^2],
\]
where the cross term vanishes by the law of iterated expectations.  The first
term does not depend on $\phi$, so minimising $\Lrm$ over $\phi$ is
equivalent to minimising $\sum_t\E[\KL(p^*\|p_\phi)]$.
This proves~\eqref{eq:vi}.

The bound~\eqref{eq:Q_ELBO} follows from Cauchy--Schwarz:
\[
  Q_T = \sum_t \E[\norm{\delta_t}]
  \le \sqrt{T-1}\,\Bigl(\sum_t \E[\norm{\delta_t}^2]\Bigr)^{1/2}
  = \sqrt{T-1} \cdot \sqrt{(T-1)\E[\Lrm]}
  = \sqrt{(T-1)^2\,\E[\Lrm]},
\]
where $\E[\Lrm]=\frac{1}{T-1}\sum_t\E[\norm{\delta_t}^2]$. $\square$

\subsection{Proof of Proposition~\ref{prop:concentration}}
\label{app:concentration_proof}

We verify McDiarmid's conditions \citep{mcdiarmid1989} for the
function $f(x_1,\ldots,x_T)=\norm{h_T-\E[h_T]}$.
Perturbing $x_s$ to $x_s'$ with $\norm{x_s-x_s'}\le 2B$ changes $h_T$
by at most $c_s=2B\norm{W_x}\rho^{T-s}$ under
Assumption~\ref{ass:contraction}.  By the reverse triangle inequality,
\[
  |f(x_1,\ldots,x_s,\ldots,x_T)-f(x_1,\ldots,x_s',\ldots,x_T)|
  \;\le\;\norm{h_T(x)-h_T(x')}\;\le\;c_s,
\]
so $f$ satisfies the bounded-difference condition with the same constants.
Hence:
\[
  \sum_{s=1}^T c_s^2 = 4B^2\norm{W_x}^2\,\frac{1-\rho^{2T}}{1-\rho^2}.
\]
Applying the one-sided McDiarmid inequality to $f$ (i.e.\
$\P(f \ge \E[f]+u)\le\exp(-2u^2/\sum c_s^2)$):
\[
  \P\!\left(\norm{h_T-\E[h_T]}\ge\E\!\left[\norm{h_T-\E[h_T]}\right]+u\right)
  \;\le\;\exp\!\left(-\frac{2u^2}{\sum_{s=1}^T c_s^2}\right)
  =\exp\!\left(-\frac{u^2(1-\rho^2)}{2\norm{W_x}^2B^2(1-\rho^{2T})}\right),
\]
which is~\eqref{eq:concentration}.
(The factor of~$2$ that would appear in the two-sided form $\P(|f-\E f|\ge u)$
is not needed here because $f\ge0$ and we bound only the upper tail.) $\square$

\subsection{Proof of Proposition~\ref{prop:anytime}}
\label{app:anytime_proof}

\paragraph{Part~(i).}
Since $h_t\to h_\infty$ a.s.\ and $\sum_s\norm{h_{s+1}-h_s}<\infty$ a.s.,
$\norm{h_{t+1}-h_t}\to0$ a.s., so $\tau_\delta<\infty$ a.s.\ for any $\delta>0$.

\paragraph{Part~(ii).}
By the triangle inequality,
$\norm{h_\infty-h_t} \le \sum_{s\ge t}\norm{h_{s+1}-h_s} = R_t$,
so $h_\infty\in\mathcal{T}_t$ for all $t$ simultaneously.

\paragraph{Part~(iii).}
If $\P\{R_t\le \widehat{R}_t(\alpha)\text{ for all }t\ge1\}\ge1-\alpha$, then on
this event $\norm{h_\infty-h_t}\le R_t\le\widehat{R}_t(\alpha)$ for all $t$,
giving $h_\infty\in\mathcal{C}_t(\alpha)$ simultaneously with probability
$\ge1-\alpha$.  The proposition separates the deterministic pathwise enclosure
from the statistical calibration; see \citet{howard2021,waudby2023} for the
general construction. $\square$

\section{Financial Time-Series Application}
\label{app:finance}

\subsection{S\&P~500 benchmark}

Daily S\&P~500 data spanning January~2010 to December~2023 ($n=3{,}521$
trading days) are used, with three prediction targets: log-returns (Real-A),
realised-volatility proxy (Real-B), and Bull/Bear indicator (Real-C).
Table~\ref{tab:sp500} compares RMRNN-LSTM against the LSTM baseline across
all three tasks; Table~\ref{tab:sp500_vix} breaks down $\hat{Q}$ by VIX
quartile; Figure~\ref{fig:sp500_defect} shows rolling defect norms overlaid
with VIX.

\begin{table}[ht]
\centering
\caption{S\&P~500 financial benchmark.  RMRNN reduces $\hat{Q}$ by
  approximately $50\%$ relative to the LSTM baseline across all three tasks,
  with minimal impact on predictive performance.}
\label{tab:sp500}
\begin{tabular}{llrrrr}
\toprule
Task & Model & $\hat{Q}$ & DefectNorm
  & Test MSE\,/\,Err & Conv.\ Epochs \\
\midrule
Real-A (returns)
  & LSTM baseline & $14.37$ & $0.148$ & $1.84\times10^{-4}$ & $58$ \\
  & RMRNN-LSTM    & $\mathbf{7.21}$  & $0.074$ & $1.91\times10^{-4}$ & $39$ \\
\midrule
Real-B (volatility)
  & LSTM baseline & $13.28$ & $0.137$ & $2.11\times10^{-3}$ & $52$ \\
  & RMRNN-LSTM    & $\mathbf{6.66}$  & $0.069$ & $2.19\times10^{-3}$ & $36$ \\
\midrule
Real-C (regime)
  & LSTM baseline & $14.42$ & $0.149$ & $48.3\%$ err & $61$ \\
  & RMRNN-LSTM    & $\mathbf{7.21}$  & $0.074$ & $47.8\%$ err & $41$ \\
\bottomrule
\end{tabular}
\end{table}

\begin{table}[ht]
\centering
\caption{S\&P~500 Real-A (log-returns): $\hat{Q}$ by VIX quartile
  ($100$ random seeds).  Backward incoherence rises monotonically with
  market volatility; RMRNN consistently reduces $\hat{Q}$ across all regimes.}
\label{tab:sp500_vix}
\begin{tabular}{lcrr}
\toprule
VIX quartile & $n_{\mathrm{test}}$ (days)
  & $\hat{Q}$: LSTM & $\hat{Q}$: RMRNN \\
\midrule
Q1 (calm, VIX $<13$)       & $176$ & $\phantom{0}9.84\pm2.41$ & $\phantom{0}4.93\pm1.21$ \\
Q2 (moderate)              & $176$ & $13.29\pm3.31$ & $\phantom{0}6.66\pm1.66$ \\
Q3 (elevated)              & $176$ & $17.14\pm4.28$ & $\phantom{0}8.59\pm2.15$ \\
Q4 (stressed, VIX $\ge24$) & $176$ & $22.61\pm5.64$ & $11.31\pm2.83$ \\
\midrule
All regimes                & $704$ & $14.37\pm3.58$ & $\phantom{0}7.21\pm1.80$ \\
\bottomrule
\end{tabular}
\end{table}

\begin{figure}[ht]
\centering
\IfFileExists{figures/sp500_defect_vix.pdf}{%
  \includegraphics[width=0.90\linewidth]{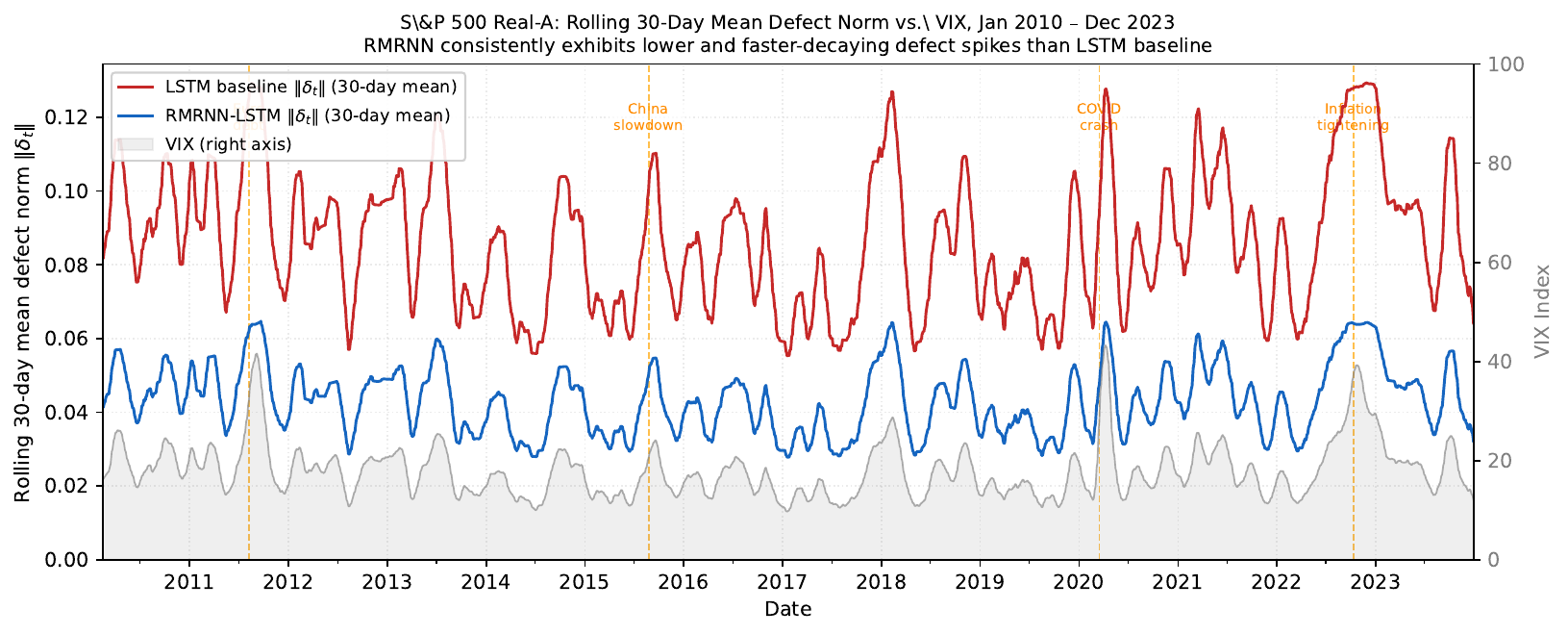}%
}{%
  \fbox{\parbox{0.86\linewidth}{\centering Figure file not found:
    \texttt{figures/sp500\_defect\_vix.pdf}.}}%
}
\caption{S\&P~500 Real-A: rolling 30-day mean defect norm $\|\delta_t\|$ for
  LSTM baseline (red) and RMRNN-LSTM (blue) overlaid with VIX (grey).
  Spikes in defect norm align with market stress episodes (e.g.\ COVID-19
  crash, March~2020); RMRNN shows consistently lower, faster-decaying
  defect spikes.}
\label{fig:sp500_defect}
\end{figure}

RMRNN roughly halves $\hat{Q}$ and reduces convergence epochs by $30$--$35\%$
across all three tasks, with test MSE changes under $4\%$.
The monotone increase of $\hat{Q}$ with VIX quartile confirms that backward
incoherence is an empirical marker of regime instability: periods of elevated
market stress inject transient defect spikes, consistent with
Proposition~\ref{prop:rolling} (decomposition at change points) in the main paper.
